\documentclass[english,10pt,journal]{IEEEtran}
\usepackage{color,graphicx,amsmath,amssymb,amsthm,epsfig,mathrsfs,cite,bm,enumerate}
\usepackage{array,stfloats}
\usepackage{verbatim}
\usepackage{mathtools}

\DeclarePairedDelimiter\floor{\lfloor}{\rfloor}
\usepackage{color,soul}
\usepackage{multirow}
\makeatletter
\newcommand*{\rom}[1]{\expandafter\@slowromancap\romannumeral #1@}
\makeatother
\newcommand{\squeezeup}{\vspace{-2.5mm}}
\newcommand{\upperRomannumeral}[1]{\uppercase\expandafter{\romannumeral#1}}

\newtheorem{lemma}{Lemma}{}
  \newtheorem{thm}{Theorem}

\theoremstyle{remark} \newtheorem{remark}{Remark}
\usepackage{amssymb}
\usepackage{float}
\usepackage{tikz}
\usetikzlibrary{trees}
\usepackage{algorithm}
\usepackage{algpseudocode}
\usepackage{lipsum}
\usepackage[lofdepth,lotdepth]{subfig}
\usepackage{enumerate}
\newtheorem{theorem}{Theorem}

\newtheorem{corollary}[theorem]{Corollary}

\title{Signal and Noise Statistics Oblivious  Sparse Reconstruction using OMP/OLS.} 

\author{Sreejith Kallummil,  \hspace{0cm} Sheetal Kalyani  \\
 Department of Electrical Engineering \\  Indian Institute of Technology Madras\\
  Chennai, India 600036 \\
  \{ee12d032,skalyani\}@ee.iitm.ac.in
  }

\begin{document}
\maketitle
\begin{abstract}
Orthogonal matching pursuit (OMP) and orthogonal least squares (OLS)  are widely used for  sparse signal reconstruction in under-determined linear regression problems. The performance of these compressed sensing (CS) algorithms depends crucially on the \textit{a priori} knowledge of either the sparsity of the signal ($k_0$)  or noise variance ($\sigma^2$). Both $k_0$ and $\sigma^2$  are unknown in general and extremely difficult to estimate in under determined models. This limits the application of OMP and OLS in many practical situations. In this article, we develop two computationally efficient frameworks namely TF-IGP and RRT-IGP  for using OMP and OLS  even when $k_0$ and $\sigma^2$  are unavailable.   Both TF-IGP and RRT-IGP are analytically shown to accomplish successful sparse recovery under the same set of restricted isometry conditions on the design matrix required for OMP/OLS with \textit{a priori} knowledge of $k_0$ and $\sigma^2$. Numerical simulations also indicate a highly competitive performance of TF-IGP and RRT-IGP in comparison to OMP/OLS with \textit{a priori} knowledge of $k_0$ and $\sigma^2$.
\end{abstract}

\section{Introduction}
C\footnote{This article  is a substantial revision of  an earlier article \cite{kallummil2017tuning}   titled "Tuning Free Orthogonal Matching Pursuit" submitted to arXiv with ID arXiv:1703.05080. A substantial portion of \cite{kallummil2017tuning} is dropped and entirely new algorithms and analyses are included in this article. }onsider the linear regression model ${\bf y}={\bf X}\boldsymbol{\beta}+{\bf w}$, where ${\bf X} \in \mathbb{R}^{n \times p}$ , $n<p$ is a known design matrix, ${\bf w}$ is the noise vector and ${\bf y}$ is the observation vector. Since $n<p$, the design matrix is  rank deficient, i.e., $rank({\bf X})< p$.   Further, the columns of ${\bf X}$ are normalised to have unit Euclidean $(l_2)$ norm.  The vector $\boldsymbol{\beta}\in \mathbb{R}^{ p}$ is sparse, i.e., the support of $\boldsymbol{\beta}$ given by $\mathcal{I}=supp(\boldsymbol{\beta})=\{k:\boldsymbol{\beta}_k\neq 0\}$ has cardinality $k_0=|\mathcal{I}|\ll p$. The noise vector ${\bf w}$ is assumed to have  a Gaussian distribution with mean ${\bf 0}_n$ and covariance $\sigma^2{\bf I}_n$, i.e., ${\bf w} \sim \mathcal{N}({\bf 0}_n,\sigma^2{\bf I}_n)$ or ${\bf w}$ is assumed to be  $l_2$ bounded, i.e., $\|{\bf w}\|_2\leq \epsilon_2$. The signal to noise ratio (SNR) in this regression model is defined  as  SNR$=\frac{\|{\bf X\boldsymbol{\beta}}\|_2^2}{n\sigma^2}$ for Gaussian noise and SNR=$\frac{\|{\bf X\boldsymbol{\beta}}\|_2^2}{\epsilon_2^2}$ for $l_2$ bounded noise.  In this article, we consider the following  two  problems  in the context of recovering sparse vectors.

P1). Estimate ${\boldsymbol{\beta}}$ with the objective of minimizing the normalized mean squared error NMSE$(\hat{{\boldsymbol{\beta}}})=\frac{\mathbb{E}(\|{\boldsymbol{\beta}}-\hat{{\boldsymbol{\beta}}}\|_2^2)}{\|\boldsymbol{\beta}\|_2^2}$.

P2). Estimate  ${\boldsymbol{\beta}}$ with the objective of minimizing  support recovery error $PE(\hat{\boldsymbol{\beta}})=\mathbb{P}(\hat{\mathcal{I}}\neq \mathcal{I})$, where $\hat{\mathcal{I}}=supp(\hat{\boldsymbol{\beta}})$.

  These problems  known in the  signal processing community under the compressed sensing\cite{eldar2012compressed} paradigm has large number of applications like face recognition\cite{eldar2012compressed}, direction of arrival estimation\cite{single_snap}, MIMO detection\cite{MLOMP} etc.   A number of algorithms like least absolute shrinkage and selection operator (LASSO)\cite{tropp2006just}, Dantzig selector (DS)\cite{candes2007dantzig}, subspace pursuit (SP)\cite{subspacepursuit}, compressive sampling matching pursuit (CoSaMP)\cite{cosamp},  OMP \cite{tropp2004greed,cai2011orthogonal,OMP_wang,extra,omp_sharp,omp_necess}, OMP with replacement (OMPR)\cite{prateek}, OLS\cite{ERC-OMP,OLSarxiv} etc. are proposed to solve the above mentioned problems. For the efficient performance of these algorithms, a number of tuning parameters (or hyper parameters) need to be fixed. These tuning parameters require \textit{a priori} knowledge of signal parameters like sparsity $k_0$ or noise statistics like $\{\sigma^2,\epsilon_2\}$ or both. Further, a level of user subjectivity is often required  even when these statistics are known \textit{a priori}. \\
{\bf Definition 1:-} A CS algorithm \textit{Alg} is called signal and noise statistic oblivious (SNO) if the efficient performance  of \textit{Alg} does not require \textit{a priori} knowledge of signal or noise parameters.  Further, a SNO CS algorithm \textit{Alg} is  called tuning free (TF), if the optimal performance  of \textit{Alg} does not depend on any user defined hyper parameters. 

Algorithms like LASSO, DS etc. are  noise statistic dependent in the sense that the optimal choice of hyper parameters in these algorithms require knowledge of $\{\sigma^2,\epsilon_2\}$. Greedy algorithms like SP, CoSaMP, OMPR etc. are signal statistic dependent in the sense that they require \textit{a priori} knowledge of $k_0$ for their optimal performance. Algorithms  like OMP, OLS etc. can be operated as  either  signal dependent  with $k_0$ as input or  noise dependent  with $\{\sigma^2,\epsilon_2\}$ as input. However, neither $k_0$ nor $\sigma^2$ are \textit{a priori} known in most practical applications. Further, unlike the case of full rank linear regression models $(n>p)$ where one can readily estimate $\sigma^2$ using the maximum likelihood  estimator $\sigma^2_{ML}=\frac{\|\left({\bf I}_n-{\bf X}({\bf X}^T{\bf X})^{-1}{\bf X}^T\right){\bf y}\|_2^2}{n-p}$, estimating $\sigma^2$ in under determined linear regression models $(n<p)$ is extremely difficult\cite{giraud2012}. Hence, signal/noise statistic dependent algorithms are not useful in most practical applications where the user is oblivious to signal and noise statistics. This led to the development of many SNO CS algorithms recently.
\subsection{SNO algorithms: Prior art.}
  A significant breakthrough in the design of SNO CS algorithms  is the development of square root LASSO (sq-LASSO)\cite{sqlasso}.  The optimal NMSE performance of sq-LASSO does not require \textit{a priori} knowledge of $\{\sigma^2,\epsilon_2\}$ thereby overcoming a major drawback of LASSO. However, the choice of hyper parameter in sq-LASSO is still subjective with few guidelines. On the contrary, the sparse iterative covariance-based estimation \textit{aka} SPICE\cite{spice,spice_like,spicenote} is a  tuning free CS algorithm. Both sq-LASSO, SPICE and their derivatives are based on convex optimization and hence are computationally complex.  Greedy algorithms like OMP, SP etc.  have significantly lower complexity when compared to  sq-LASSO, SPICE etc. This motivated  the low complexity  PaTh framework  in  \cite{vats2014path} that can use OMP, SP etc. in a SNO fashion. PaTh is shown to have nice asymptotic properties. However, PaTh requires the setting of a parameter $c>0$, the choice of which in finite $n$ and $p$ is subjective. Hence, PaTh is a SNO algorithm but not TF. Unfortunately, PaTh performs poorly in many SNR-sparsity regimes (see Section.\rom{7}).  To summarize, SNO algorithms like SPICE that can perform efficiently are computationally  complex, whereas, low complexity SNO frameworks like PaTh performs highly sub-optimally. This motivates the OMP/OLS based SNO frameworks  developed in this article that can deliver a highly competitive  performance with significantly lower complexity in comparison with SPICE, sq-LASSO etc. 
\subsection{Contribution of this article.}
 This article propose two novel computationally efficient frameworks for using a particular class of greedy algorithms which we call incremental greedy pursuits (IGP) in a SNO fashion, i.e., without knowing $k_0$ or $\{\sigma^2,\epsilon_2\}$ \textit{a priori}. IGP includes popular algorithms like OMP, OLS etc.  The first framework called tuning free IGP (TF-IGP)    is  devoid of any tuning parameters. Both analytical results and numerical simulations indicates a degraded performance of TF-IGP when the dynamic range of $\boldsymbol{\beta}$ given by $DR(\boldsymbol{\beta})={\underset{j \in \mathcal{I}}{\max}|\boldsymbol{\beta}_j|}/{\underset{j \in \mathcal{I}}{\min}|\boldsymbol{\beta}_j|}$ is high. Hence, TF-IGP framework is more suited for applications like \cite{MLOMP} where $DR(\boldsymbol{\beta})\approx 1$.  This motivated the development of  residual ratio threshold IGP (RRT-IGP) framework which can perform efficiently even when $DR(\boldsymbol{\beta})$ is high. RRT-IGP depends very weakly  on a tuning parameter which can be set independently of $k_0$ or $\{\sigma^2,\epsilon_2\}$. Hence, RRT-IGP is SNO, but not TF.   Both TF-IGP and RRT-GP are analytically shown to recover  the true support $\mathcal{I}$  under the same set of  conditions on the matrix ${\bf X}$ for IGP to recover the true support if $k_0$ or $\{\sigma^2,\epsilon_2\}$ are known \textit{a priori}. Unlike PaTh framework, our analysis of TF-IGP and RRT-IGP are finite sample in nature  and hence more general.  Numerical simulations  indicate that the   performance of TF-IGP (when $DR(\boldsymbol{\beta})\approx 1$) and RRT-IGP closely matches the performance of   IGP with \textit{a priori} knowledge of $k_0$ or $\{\sigma^2,\epsilon_2\}$ throughout the moderate to high SNR regime. Even in the low SNR regime, the performance gap between TF-IGP/RRT-IGP and IGP with \textit{a priori} knowledge of $k_0$ or $\{\sigma^2,\epsilon_2\}$ are not  significant. Further, we  analytically and empirically demonstrate that  TF-IGP/RRT-IGP  can outperform IGP with \textit{a priori} knowledge of $k_0$ in certain sparsity and SNR regimes. By providing a performance comparable to that of  IGP which has \textit{a priori} knowledge of $k_0$ and $\{\sigma^2,\epsilon_2\}$, TF-IGP/RRT-IGP can extend the scope of  IGP  to applications where $k_0$ or $\{\sigma^2,\epsilon_2\}$ are not known \textit{a priori}. 
\subsection{Notations used.}
{$\mathbb{E}()$  and $\mathbb{P}()$ represents the expectation and probability respectively. $A|B$ denotes the event $A$ conditioned on the event $B$.  $\|{\bf x}\|_q=\left( \sum\limits_{k=1}^p|{\bf x}_k|^q\right)^{\frac{1}{q}} $ is the $l_q$ norm of ${\bf x}\in \mathbb{R}^{p}$.
${\bf 0}_n$ is the $n\times 1$ zero vector and ${\bf I}_n$ is the $n\times n$ identity matrix. ${\bf X}^T$ is the transpose and ${\bf X}^{-1}$ is the inverse of ${\bf X}$.}
$col({\bf X})$ is the column space of ${\bf X}$. ${\bf X}^{\dagger}=({\bf X}^T{\bf X})^{-1}{\bf X}^T$ is the  Moore-Penrose pseudo inverse of ${\bf X}$. ${\bf P}_{\bf X}={\bf X}{\bf X}^{\dagger}$ is the projection matrix onto $col({\bf X})$.  ${\bf X}_{\mathcal{J}}$ denotes the sub-matrix of ${\bf X}$ formed using  the columns indexed by $\mathcal{J}$. ${\bf X}_{i,j}$ is the $[i,j]^{th}$ entry of ${\bf X}$. If ${\bf X}$ is clear from the context, we use the shorthand ${\bf P}_{\mathcal{J}}$ for ${\bf P}_{{\bf X}_{\mathcal{J}}}$. ${\bf a}_{\mathcal{J}}$  denotes the  entries of ${\bf a}$ indexed by $\mathcal{J}$.   $\chi^2_j$ denotes a central chi square random variable (R.V) with $j$ degrees of freedom (d.o.f). $\mathbb{B}(a,b)$ denotes a Beta R.V with parameters $a$ and $b$\cite{ravishanker2001first}.   ${\bf a}\sim{\bf b}$ implies that ${\bf a}$ and ${\bf b}$ are identically distributed.    $[p]$ denotes the set $\{1,\dotsc,p\}$. $\floor{x}$ denotes the floor function. $\phi$ represents the null set. For any two index sets $\mathcal{J}_1$ and $\mathcal{J}_2$, the set difference  $\mathcal{J}_1/\mathcal{J}_2=\{j:j \in \mathcal{J}_1\& j\notin  \mathcal{J}_2\}$. For any index set $\mathcal{J}\subseteq [p]$, $\mathcal{J}^C$ denotes the  complement of $\mathcal{J}$ with respect to $[p]$. $f(n)=O(g(n))$ iff $\underset{n \rightarrow \infty}{\lim}\frac{f(n)}{g(n)}<\infty$. TF-Alg/RRT-Alg represents the application of a particular algorithm `Alg' in the TF-IGP/RRT-IGP framework. $Alg({\bf y},{\bf X},k)$ represents any CS algorithm \textit{Alg} with inputs ${\bf y},{\bf X}$ and sparsity level $k$ that produce a support estimate $\hat{\mathcal{I}}_k=Alg({\bf y},{\bf X},k)$ of cardinality $|\hat{\mathcal{I}}_k|= k$ as output. $\boldsymbol{\beta}_{max}=\underset{j \in \mathcal{I}}{\max}|\boldsymbol{\beta}_j|$ and  $\boldsymbol{\beta}_{min}=\underset{j \in \mathcal{I}}{\min}|\boldsymbol{\beta}_j|$ denotes the maximum and minimum  non zero values in $\boldsymbol{\beta}$.  $DR(\boldsymbol{\beta})=\boldsymbol{\beta}_{max}/\boldsymbol{\beta}_{min}$ is the dynamic range of $\boldsymbol{\beta}$.
 \subsection{ Organization of this article:-} Section \rom{2}  discuss the concept of restricted isometry constants (RIC). Section \rom{3} discuss IGP. Section \rom{4} and \rom{5} present the TF-GP and  RRT-IGP frameworks. Section \rom{6} relates the performance of TF-OMP/RRT-OMP and  OMP  with \textit{a priori} knowledge of $k_0$.    Section \rom{7} present  numerical simulations. 
\section{ Qualifiers for CS matrices.}
 Estimating sparse vectors in under determined regression models is ill posed in general. It is known that the  efficient estimation of  $\boldsymbol{\beta}$ is possible when  ${\bf X}$  satisfies  regularity conditions like restricted isometry property (RIP)\cite{eldar2012compressed,OMP_wang},  exact recovery condition (ERC)\cite{tropp2004greed}, mutual incoherence condition (MIC)\cite{cai2011orthogonal} etc.  The analysis based on RIP is more popular in literature. Hence, this article will focus on RIP based analysis. \\
{\bf Definition 2:-} RIC of order $k$ denoted  by $\delta_k$ is defined as the smallest value of $0\leq \delta\leq 1$ that satisfies
\begin{equation}
(1-\delta)\|{\bf b}\|_2^2\leq \|{\bf Xb}\|_2^2\leq (1+\delta)\|{\bf b}\|_2^2 
\end{equation}
for all $k$-sparse ${\bf b}\in \mathbb{R}^p$\cite{eldar2012compressed}. ${\bf X}$ satisfy RIP of order $k$ if $\delta_k<1$. \\
 Lemma 1 summarizes certain useful properties of  RIC $\delta_k$. 
\begin{lemma}\label{eigenvalue}
 Let the matrix  ${\bf X}$ satisfy RIP of order $k$. Then the following results hold true.\cite{OMP_wang,subspacepursuit} \\
 a). $\delta_{k_1}\leq \delta_{k_2}$ whenever $k_1\leq k_2$. \\
 b).$\dfrac{1}{1+\delta_{k}}\|{\bf a}\|_2\leq \|\left({\bf X}_{\mathcal{J}}^T{\bf X}_{\mathcal{J}}\right)^{-1}{\bf a}\|_2\leq\dfrac{1}{ 1-\delta_k}\|{\bf a}\|_2$, $\forall \mathcal{J}\subset[p]$ with $|\mathcal{J}|\leq k$. \\
 c). $ \|{\bf X}_{\mathcal{J}}^{\dagger}{\bf a}\|_2\leq\dfrac{1}{\sqrt{ 1-\delta_k}}\|{\bf a}\|_2$, $\forall \mathcal{J}\subset[p]$ with $|\mathcal{J}|\leq k$. \\
 d). If $\mathcal{J}_1\cap\mathcal{J}_2=\phi$ and $|\mathcal{J}_1\cup\mathcal{J}_2|\leq k$, then $\|{\bf X}_{\mathcal{J}_1}^T {\bf X}_{\mathcal{J}_2}{\bf a}\|_2\leq \delta_k\|{\bf a}\|_2$. \\
 e).  If $\mathcal{J}_1\cap\mathcal{J}_2=\phi$ and $|\mathcal{J}_1\cup\mathcal{J}_2|\leq k$, then $(1-\delta_k)\|{\bf a}\|_2^2\leq \|({\bf I}_n-{\bf P}_{\mathcal{J}_1}){\bf X}_{\mathcal{J}_2}{\bf a}\|_2^2\leq (1+\delta_k)\|{\bf a}\|_2^2$ [Lemma 4, \cite{omp_sharp}].
 \end{lemma}
 \section{Incremental Greedy Pursuits}
The proposed SNO frameworks are based on a particular class of greedy algorithms   called incremental greedy pursuits (IGP) which is formally defined  below. \\
{\bf Definition 3:- }Consider  a CS algorithm \textit{Alg} with inputs ${\bf y},{\bf X}$ and sparsity level $k$ producing a support estimate $\hat{\mathcal{I}}_k=Alg({\bf y},{\bf X},k)$ of cardinality $|\hat{\mathcal{I}}_k|= k$ as output. \textit{Alg} is an IGP  iff it satisfies   conditions (A1) and (A2) at all SNR. Let $\hat{\mathcal{I}}_{1},\dotsc, \hat{\mathcal{I}}_{K}$ be a sequence of support estimates produced by $Alg({\bf y},{\bf X},k)$ as $k$ varies from $k=1$ to $k= K$. \\
A1).{\bf Monotonicity:} $\hat{\mathcal{I}}_{k_1} \subset \hat{\mathcal{I}}_{k_2}$  whenever $k_1<k_2\leq K$. \\
A2).{\bf Reproducibility:}  For any $k<K$, the output of  $Alg\left(({\bf I}_n-{\bf P}_{\hat{\mathcal{I}}_k}){\bf y},{\bf X},j\right) $ for $j=1,\dotsc,K-k$ should be $\hat{\mathcal{I}}_{k+1}/\hat{\mathcal{I}_k},\dotsc, \hat{\mathcal{I}}_{K}/\hat{\mathcal{I}_k}$. \\

 A1) implies that $\{\hat{\mathcal{I}}_{k}\}_{k=1}^K$ can be written as ordered sets $\hat{\mathcal{I}}_1=\{t_1\}$, $\hat{\mathcal{I}}_2=\{t_1,t_2\}$ and $\hat{\mathcal{I}}_K=\{t_1,t_2, \dotsc,t_K\}$. Reproducibility  property A2) implies that output of IGP with input $({\bf I}_n-{\bf P}_{\hat{\mathcal{I}}_k}){\bf y}$ and  sparsity level $j=1,2,\dotsc K-k$ will be of the form $\{t_{k+1}\}, \{t_{k+1}, t_{k+2}\},\dotsc, \{t_{k+1},t_{k+2},\dotsc,t_{K}\}$.  A2)  implies that the new index selected at  sparsity level $j= k+1$ depends on ${\bf y}$ only through the residual at sparsity level $k$, i.e., $({\bf I}_n-{\bf P}_{\hat{\mathcal{I}}_k}){\bf y}$.   A2) also  means that the  output of IGP for sparsity levels $j\geq k+1$  can be recreated with the residual in the $k^{th}$ level, i.e., $({\bf I}_n-{\bf P}_{\hat{\mathcal{I}}_k}){\bf y}$ as input.  We next consider some positive and negative examples for IGP. 
 \subsection{Popular IGP: Algorithms like OMP, OLS  etc.}
OMP and OLS  described in TABLE \ref{tab:omp}  are among the most popular algorithms in CS literature.  OMP starts with a null model and add that column index to the current support  that is the most correlated with the current residual.  OLS like OMP also starts with the null model, however, OLS add the column index $t^k$ that will result in the maximum reduction in the residual error $\|{\bf r}^{(k)}\|_2$.  For OMP/OLS to behave like $Alg({\bf y},{\bf X},k)$, i.e., to return a support estimate of cardinality $k$, one should run precisely $k$ iterations in TABLE \ref{tab:omp}, i.e., $\hat{\mathcal{I}}_k=Alg({\bf y},{\bf X},k)$ is equal to $\mathcal{J}_k$, the  support estimate of OMP/OLS after the $k^{th}$ iteration. Since, $\hat{\mathcal{I}}_k= \hat{\mathcal{I}}_{k-1}\cup t^k$, OMP/OLS are monotonic. Note that the index selected in the $k^{th}$ iteration of OMP/OLS, i.e.,  $t^k$ depends only on  $\mathcal{J}_{k-1}$ which is same as $\hat{\mathcal{I}}_{k-1}=Alg({\bf y},{\bf X},k-1)$. Hence,  OMP/OLS   satisfies the  reproducibility condition of IGP also.   
\begin{table}
\begin{tabular}{|l|}
\hline
  {\bf Step 1:-} Initialize the residual ${\bf r}^{(0)}={\bf y}$. $\hat{\boldsymbol{\beta}}={\bf 0}_p$,\\ 	\ \ \ \ \ \ \ \ \ \ \   Support estimate ${\mathcal{J}_0}=\phi$, Iteration counter $k=1$; \\
  {\bf Step 2:-} Update support estimate: ${\mathcal{J}_k}={\mathcal{J}_{k-1}}\cup t^k$ \\
\ \ \ \ \ \ \ \ \ \ \ \ OMP: $t^k=\underset{t \notin {\mathcal{J}_{k-1}}}{\arg\max}|{\bf X}_t^T{\bf r}^{(k-1)}|.$ \\

\ \ \ \ \ \ \ \ \ \ \ \ OLS: $t^k=\underset{t \notin {\mathcal{J}_{k-1}}}{\arg\min}\|({\bf I}_n-{\bf P}_{\mathcal{J}_{k-1}\cup t}){\bf y}\|_2 .$\\
 
  {\bf Step 4:-} Estimate $\boldsymbol{\beta}$ using current support: $\hat{\boldsymbol{\beta}}(\mathcal{J}_k)={\bf X}_{\mathcal{J}_k}^{\dagger}{\bf y}$. \\
  {\bf Step 5:-} Update residual: ${\bf r}^{(k)}={\bf y}-{\bf X}\hat{\boldsymbol{\beta}}=({\bf I}_n-{\bf P}_{\mathcal{J}_k}){\bf y}$. \\
  {\bf Step 6:-} Increment $k$. $k \leftarrow k+1$. \\
  {\bf Step 7:-} Repeat Steps 2-6, until stopping condition (SC)  is  met. \\
  {\bf Output:-} $\hat{\mathcal{I}}=\mathcal{J}_k$ and $\hat{\boldsymbol{\beta}}$. \\
 \hline
\end{tabular}
\caption{ OMP and OLS algorithms.}
\label{tab:omp}
\end{table}
\begin{remark}
Algorithms like SP\cite{subspacepursuit}, CoSaMP\cite{cosamp}, OMPR\cite{prateek} etc.  returns a support estimate with sparsity $k$ when used as $\textit{Alg}({\bf y},{\bf X},k)$. However, these algorithms do not exhibit the monotonicity and reproducibility of supports required for IGP.
\end{remark}
\subsection{Stopping conditions (SC) for OMP/OLS.}
Most of the theoretical properties of  OMP/OLS  are derived assuming  \textit{a priori} knowledge of true sparsity level $k_0$ in which case OMP/OLS stops after exactly $k_0$ iterations\cite{tropp2004greed,omp_necess,OLSarxiv}. When $k_0$ is not known, one has to rely on SC based on the properties of the residual ${\bf r}^{(k)}=({\bf I}_n-{\bf P}_{\hat{\mathcal{I}}_k}){\bf y}$ in Step 4 of TABLE \ref{tab:omp} as $k$ varies. Such residual based SC has attained a  level of maturity in the case of OMP. For example, OMP can be stopped in Gaussian noise if $\|{\bf r}^{(k)}\|_2<\sigma\sqrt{n+2\sqrt{n\log(n)}}$ \cite{cai2011orthogonal,omp_rip_noise} or $\|{\bf X}^T{\bf r}^{(k)}\|_{\infty}<\sigma \sqrt{2\log(p)}$\cite{cai2011orthogonal}. In the case of $l_2$ bounded noise, OMP can be stopped if $\|{\bf r}^{(k)}\|_2\leq\epsilon_2$.   Likewise, \cite{resdif}  suggested a SC based on the  residual difference ${\bf r}^{(k)}-{\bf r}^{(k-1)}$.  A generalized likelihood ratio based SC is developed in \cite{Xiong2014}.  All these residual based  SC requires the \textit{a priori} knowledge of $\sigma^2$. Knowing $k_0$ or $\{\sigma^2,\epsilon_2\}$ \textit{a priori} is not possible in many practical problems and estimating $\{\sigma^2,\epsilon_2\}$ when $n<p$  is extremely difficult. This makes OMP/OLS useless in many applications where $k_0$ or $\{\sigma^2,\epsilon_2\}$ are unknown \textit{a priori}. 
\begin{remark}
{For algorithms  $Alg \in \{\text{OMP},\text{OLS}\}$, we use the shorthand $Alg_{k_0}$ to represent the situation when $k_0$ is known \textit{a priori} and OMP/OLS run $k_0$ iterations. Likewise, $Alg_{\sigma^2}$ or $Alg_{\epsilon_2}$ represents the situation when $\sigma^2$ or $\epsilon_2$ are known \textit{a priori}  and  iterations in OMP/OLS are continued until $\|{\bf r}^{(k)}\|_2\leq \sigma\sqrt{n+2\sqrt{n\log(n)}}$  or $\|{\bf r}^{(k)}\|\leq \epsilon_2$.}
\end{remark}
\section{Tuning Free Incremental Greedy Pursuits.} 
\begin{table}\centering

\begin{tabular}{|l|}
\hline
{\bf Input:-} Observation ${\bf y}$, design matrix ${\bf X}$. Initial residue ${\bf r}^0={\bf y}$.\\
\ \ \ \ \ \ \ \ \ \ \ Repeat Steps 1-3 for $k=1:k_{max}=\floor{\frac{n+1}{2}} $. \\ \ \ \ \ \ \ \ \ \ \ \ Stop iterations before ${\bf r}^{(k)}={\bf 0}_n$ or  ${\bf X}_{\hat{\mathcal{I}}_k}$ is rank deficient.  \\
{\bf Step 1:-} Compute $\hat{\mathcal{I}_k}=Alg({\bf y},{\bf X},k)$.\\
{\bf Step 2:-} Compute the residual ${\bf r}^{(k)}=({\bf I}_n-{\bf P}_{\hat{\mathcal{I}_k}}){\bf y}$. \\
{\bf Step 3:-} Compute the residual ratio  $RR(k)=\frac{\|{\bf r}^{(k)}\|_2}{\|{\bf r}^{(k-1)}\|_2}$. \\
{\bf Step 4:-} Estimate $\hat{k}_{TF}=\underset{k=1:k_{max}}{\arg\min}RR(k)$. \\
{\bf Step 5:-} Support estimate
 $\hat{\mathcal{I}}=\hat{\mathcal{I}}_{\hat{k}_{TF}}$. \\  \ \ \ \ \ \ \ \ \ \ Signal estimate $\boldsymbol{\beta}$ as $\hat{\boldsymbol{\beta}}(\hat{\mathcal{I}})={\bf X}_{\hat{\mathcal{I}}}^{\dagger}{\bf y}$ and $\hat{\boldsymbol{\beta}}(\hat{\mathcal{I}}^C)={\bf 0}_{p-\hat{k}_{TF}}$.\\
{\bf Output:-} Support estimate $\hat{\mathcal{I}}$ and signal estimate $\hat{\boldsymbol{\beta}}$. \\
\hline
\end{tabular}
\caption{ Tuning free incremental greedy pursuits (TF-IGP).}
\label{tab:tf-omp}
\end{table}
TF-IGP outlined in TABLE \ref{tab:tf-omp} is a novel framework for using IGP when $k_0$ or $\{\sigma^2,\epsilon_2\}$ are not available.  This framework is based on the evolution of the residual ratio $RR(k)={\|{\bf r}^{(k)}\|_2}/{\|{\bf r}^{(k-1)}\|_2}$, where ${\bf r}^{(k)}=({\bf I}_n-{\bf P}_{\hat{\mathcal{I}}_k}){\bf y}$ is the residual corresponding to the support estimate $\hat{\mathcal{I}}_k=Alg({\bf y},{\bf X},k)$. {From the description of  OMP/OLS, it is clear that  the quantities $\{\hat{\mathcal{I}}_k\}_{k=1}^{k_{max}}$ and $\{{\bf r}^{(k)}\}_{k=1}^{k_{max}}$ required for the implementation of TF-IGP using OMP/OLS can be computed in a single run of OMP/OLS  with sparsity level $k_{max}$ as input.}
\subsection{Evolution of RR(k) for $k=1,\dotsc,k_{max}$ at high SNR.}
By the definition of IGP, the support estimates $\hat{\mathcal{I}}_k$ are monotonic, i.e., $\hat{\mathcal{I}}_{k_1}\subset \hat{\mathcal{I}}_{k_2} $ whenever $k_1<k_2$. This implies that the residual norms are  non decreasing with increasing $k$, i.e., $\|{\bf r}^{(k+1)}\|_2\leq \|{\bf r}^{(k)}\|_2$. Further, the iterations in TF-IGP are stopped before ${\bf r}^{(k)}={\bf 0}_n$. Hence, $RR(k)$ satisfies  $0<  RR(k)\leq  1$. The following lemma, intuitive and simple to prove is pivotal to the understanding of $RR(k)$\cite{tsp}. 
\begin{lemma}\label{Projection} Let ${\bf z}\in span({\bf X}_{\mathcal{I}})$. Then $({\bf I}_n-{\bf P}_{\hat{\mathcal{I}}_k}){\bf z}\neq {\bf 0}_n$ if $\mathcal{I} \not \subseteq \hat{\mathcal{I}}_k $ and $({\bf I}_n-{\bf P}_{\hat{\mathcal{I}}_k}){\bf z}= {\bf 0}_n$ if $\mathcal{I} \subseteq \hat{\mathcal{I}}_k  $.
\end{lemma}
Assume that there exist a $k_*\in \{k_0,\dotsc,k_{max}\}$ such that $\mathcal{I} \subseteq \hat{\mathcal{I}}_{k_*} $ for the first time, i.e., $k_*=\min\{k: \mathcal{I}\subseteq \hat{\mathcal{I}}_{k}\}$. The signal component in ${\bf y}$, i.e.,  ${\bf X}\boldsymbol{\beta}={\bf X}_{\mathcal{I}}\boldsymbol{\beta}_{\mathcal{I}}\in span({\bf X}_{\mathcal{I}})$. Hence, by Lemma \ref{Projection} and the monotonicity of $\hat{\mathcal{I}}_k$, we have $({\bf I}_n-{\bf P}_{\hat{\mathcal{I}}_{k}}){\bf X\boldsymbol{\beta}}\neq {\bf 0}_n$ for $k<k_*$ and $({\bf I}_n-{\bf P}_{\hat{\mathcal{I}}_{k}}){\bf X\boldsymbol{\beta}}= {\bf 0}_n$ for $k\geq k_*$. Thus, ${\bf r}^{(k)}=({\bf I}_n-{\bf P}_{\hat{\mathcal{I}}_{k}}){\bf y}=({\bf I}_n-{\bf P}_{\hat{\mathcal{I}}_{k}}){\bf X\boldsymbol{\beta}}+({\bf I}_n-{\bf P}_{\hat{\mathcal{I}}_{k}}){\bf w}$ for $k<k_*$, whereas, ${\bf r}^{(k)}=({\bf I}_n-{\bf P}_{\hat{\mathcal{I}}_{k}}){\bf w}$ for $k\geq k_*$.
 We next consider three regimes in the evolution of $RR(k)$. \\
{\bf Case 1:-)}. {\bf When $k<k_*$, i.e., $\hat{\mathcal{I}}_k \not \subset \mathcal{I}$:-} Both numerator $\|{\bf r}^{(k)}\|_2$ and denominator $\|{\bf r}^{(k-1)}\|_2$ in $RR(k)$ contain contributions from signal ${\bf X\boldsymbol{\beta}}$ and noise ${\bf w}$. Hence, as $\|{\bf w}\|_2\rightarrow 0$, $RR(k)\rightarrow \frac{\|({\bf I}_n-{\bf P}_{\hat{\mathcal{I}}_{k}}){\bf X\boldsymbol{\beta}} \|_2}{\|({\bf I}_n-{\bf P}_{\hat{\mathcal{I}}_{k-1}}){\bf X\boldsymbol{\beta}}\|_2} $, which is  strictly bounded away from zero.  \\
{\bf Case 2)}.{\bf When $k=k_*$, i.e., $\mathcal{I}\subseteq \hat{\mathcal{I}}_k$ for the first time:-} Numerator $\|{\bf r}^{(k_*)}\|_2$ in $RR(k_*)$ has contribution only from the noise ${\bf w}$, whereas, denominator $\|{\bf r}^{(k_*-1)}\|_2$ has contributions from both noise ${\bf w}$ and signal ${\bf X\boldsymbol{\beta}}$. Hence, $RR(k_*)= \frac{\|({\bf I}_n-{\bf P}_{\hat{\mathcal{I}}_{k_*}}){\bf w} \|_2}{\|({\bf I}_n-{\bf P}_{\hat{\mathcal{I}}_{k_*-1}})({\bf X\boldsymbol{\beta}}+{\bf w})\|_2}\rightarrow 0 $ as $\|{\bf w}\|_2\rightarrow 0$.  \\
{\bf Case 3):-} {\bf When $k>k_*$, i.e., both $ \mathcal{I} \subset \hat{\mathcal{I}}_{k} $ and $ \mathcal{I} \subset \hat{\mathcal{I}}_{k-1} $ :- } Both numerator and denominator have only noise components, i.e., $RR(k)=\frac{\|({\bf I}_n-{\bf P}_{\hat{\mathcal{I}}_{k}}){\bf w} \|_2}{\|({\bf I}_n-{\bf P}_{\hat{\mathcal{I}}_{k-1}}){\bf w}\|_2} $. This ratio is independent of the scaling of ${\bf w}$. Further, TF-IGP stops iterations before ${\bf r}^{(k)}={\bf 0}_n$. Hence, as $\|{\bf w}\|_2 \rightarrow 0$,  $RR(k)$ converges to a value in $0<RR(k)\leq 1$ strictly bounded away from zero. 

To summarize, at high or very high SNR, the minimal value of $RR(k)$ for $1\leq k\leq k_{max}$ will be attained at $k_*$, i.e., $\hat{k}_{TF}=\underset{1\leq k\leq k_{max}}{\min}RR(k)$ in TABLE \ref{tab:tf-omp} will be equal to $k_*$ at high SNR.  These observations are also numerically illustrated in Fig.1  where a typical realization of the  quantity $RR(k)$ is plotted for OMP algorithm. In all the four plots, the indexes selected by OMP are in the order $\{1,2,3\}$, i.e., $k_*=3$ and $\hat{\mathcal{I}}_3=\mathcal{I}$. In the top figure, where DR($\boldsymbol{\beta}$)=1, the minimum of $RR(k)$ is  attained at $k_*=3$ for both 10dB and 30dB SNR. In the bottom figure where DR($\boldsymbol{\beta}$) is higher, i.e., DR($\boldsymbol{\beta}$)=4,  the minimum of $RR(k)$ at 10dB SNR is attained at $k=1$, whereas, the minimum of $RR(k)$ is attained at $k=3$ for SNR=30dB. Hence, when SNR=10dB, $\hat{k}_{TF}$ underestimates $k_*$, whereas, $\hat{k}_{TF}=k_*$ at 30dB SNR. This validates the observations made above in the sense that $k_*$ is accurately estimated at high SNR. Fig.\ref{fig:tf-omp}  also point out that for signals with high $DR(\boldsymbol{\beta})$,  $\hat{k}_{TF}$ can underestimate $k_*$ even when the SNR is moderately high.

\begin{figure}[htb]
\includegraphics[width=\columnwidth]{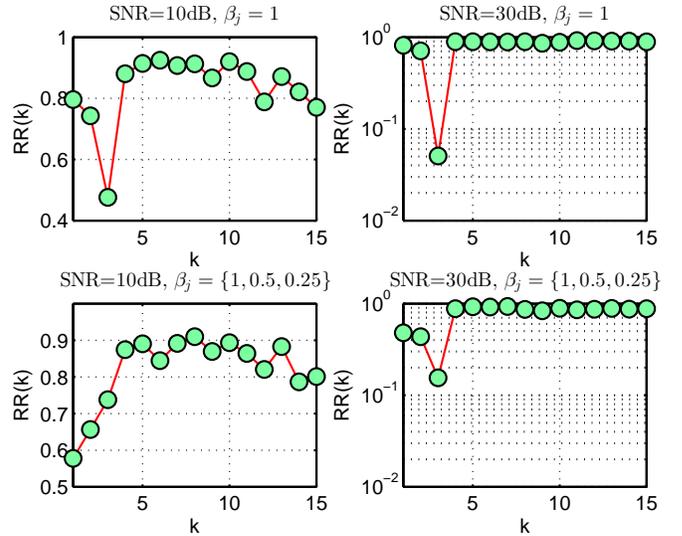}
\caption{ Evolution of $RR(k)$ for a  $32 \times 64$ matrix ${\bf X}=[{\bf I}_n,{\bf H}_n]$ described in Section \rom{7}. $\mathcal{I}=\{1,2,3\}$. (Top) $\boldsymbol{\beta}_j=1,\forall j\in \mathcal{I}$. (Bottom) $\boldsymbol{\beta}_j=(1/2)^{j-1}$ for $j=1,2$ and $3$. }
\label{fig:tf-omp}
\end{figure}

\begin{remark}
Note that the TF-IGP is designed to  estimate $k_*=\min\{k:\mathcal{I}\subseteq\hat{\mathcal{I}}_k\}$ from the sequence $\{RR(k)\}_{k=1}^{k_{max}}$.  $k_*$ will correspond to $k_0$ iff  the first $k_0$ iterations are accurate, i.e., $\hat{\mathcal{I}}_{k_0}=Alg({\bf y},{\bf X},k_0)=\mathcal{I}$. Indeed, the support estimate $\hat{\mathcal{I}}_{k_0}=Alg({\bf y},{\bf X},k_0)$ equals $\mathcal{I}$ for algorithms like OMP, OLS etc.   at high SNR under certain conditions on matrix ${\bf X}$ (see Section \rom{6}.A). In such a situation the objective of TF-IGP matches the objective of estimating $\mathcal{I}$.      When $k_*> k_0$, then $\mathcal{I}_{k_*}\supset \mathcal{I}$. Hence, if TF-IGP achieve it's stated objective of estimating $k_*$ accurately, then it will return  not the true support $\mathcal{I}$ but a superset  $\hat{\mathcal{I}}_{k_*}\supset \mathcal{I}$. Here, TF-GP includes $|\hat{\mathcal{I}}_{k_*}/ \mathcal{I}|$ number of insignificant variables in its' estimate. When $k_*>k_0$, $Alg({\bf y},{\bf X},k_0)$, i.e., $Alg_{k_0}$ itself output an erroneous support estimate, i.e., $\hat{\mathcal{I}}_{k_0}=Alg({\bf y},{\bf X},k_0)\neq \mathcal{I}$. In fact, $Alg_{k_0}$  misses  $|\mathcal{I}/ \hat{\mathcal{I}}_{k_0}|$ significant indices in $\mathcal{I}$. $\hat{\mathcal{I}}_{k_0}$  also include  $|\hat{\mathcal{I}}_{k_0}/ \mathcal{I}|$ insignificant variables. Due to this tendency of superset selection, TF-IGP can outperform $Alg_{k_0}$ in certain SNR sparsity regimes (see Section \rom{6}.C). 
\end{remark}
\subsection{Properties of RR(k) for $k>k_*$.}
{ The quantity $RR(k)$ for $k>k_*$ exhibit many interesting properties which are pivotal to the understanding of TF-IGP. As aforementioned,  ${\bf r}^{(k)}=({\bf I}_n-{\bf P}_{\hat{\mathcal{I}}_k}){\bf y}=({\bf I}_n-{\bf P}_{\hat{\mathcal{I}}_k}){\bf w}$ for $k\geq k_*$. Hence, for $k>k_*$, $RR(k)=\frac{\|({\bf I}_n-{\bf P}_{\hat{\mathcal{I}}_k}){\bf w}\|_2}{\|({\bf I}_n-{\bf P}_{\hat{\mathcal{I}}_{k-1}}){\bf w}\|_2}$. The important properties of $RR(k)$ are summarized in Theorem \ref{thm:Beta}.
\begin{thm} \label{thm:Beta}
Let the noise ${\bf w}\sim \mathcal{N}({\bf 0}_n,\sigma^2{\bf I}_n)$. RR(k) for $k>k_*$ satisfy the following properties. \\
B1). Let $F_{a,b}(x)$ be the cumulative distribution function (CDF) of a $\mathbb{B}(a,b)$ R.V and $F^{-1}_{a,b}(x)$ be its' inverse CDF. Then, $\forall \alpha>0$,  $\Gamma_{RRT}^{\alpha}=\underset{k=1,\dotsc,k_{max}}{\min}\sqrt{F_{\frac{n-k}{2},0.5}^{-1}\left(\dfrac{\alpha}{k_{max}(p-k+1)}\right)}>0$ satisfy $\mathbb{P}(\underset{k>k_*}{\min}RR(k)<\Gamma_{RRT}^{\alpha})\leq \alpha$ for all matrices ${\bf X} \in \mathbb{R}^{n\times p}$, all $\sigma^2>0$ and all algorithms $Alg$ in IGP class.\\
B2). $\underset{k>k_*}{\min}RR(k)$ for $k>k_*$ is bounded away from zero in probability, i.e., for every $\epsilon>0$, $\exists \delta>0$ such that $\mathbb{P}(\underset{k>k_*}{\min}RR(k)>\delta)\geq 1-\epsilon$.  
\end{thm} 
\begin{proof} Please see APPENDIX A.
\end{proof}
 Note that $\Gamma_{RRT}^{\alpha}$ in B1) of Theorem 1 is independent of the particular matrix ${\bf X}$, the operating SNR and the particular algorithm `Alg' in the IGP class. In other words, it is a function of matrix dimensions $n$, $p$ and  TF-IGP parameter $k_{max}$. 
 \begin{corollary} \label{corr:beta} {For a particular matrix ${\bf X}$ and Algorithm `Alg', Theorem 1 essentially implies that there exists a deterministic quantity $\Gamma_{Alg}({\bf X})>0$ which may depend on ${\bf X}$ and `Alg' such that  $\underset{k>k_*}{\min}RR(k)\geq \Gamma_{Alg}({\bf X})>0$ with a very high probability irrespective of the signal $\boldsymbol{\beta}$ and SNR.}
 \end{corollary}
\subsection{Superset and Exact  support recovery using TF-IGP.}
 In this section, we state an important theorem which  analytically establish the potential of TF-IGP to accurately estimate $k_*=\min\{k:\mathcal{I}\subseteq\hat{\mathcal{I}_k}\}$.  Theorem 1 is stated for the case of $l_2$-bounded noise, i.e., $\|{\bf w}\|_2<\epsilon_2$. Since, Gaussian vector ${\bf w}\sim \mathcal{N}({\bf 0}_n,\sigma^2{\bf I}_n)$ satisfy $\mathbb{P}\left(\|{\bf w}\|_2>\sigma\sqrt{n+2\sqrt{n\log(n)}}\right)\leq {1}/{n}$, it is also  bounded with a very high probability.  We assume that Algorithm satisfy the following conditions.  \\
{\bf Assumption 1:-} There exists  $k_{sup}\in \{k_0,\dotsc,k_{max}\}$ and $ \epsilon_{sup}>0$ such that $k_*\leq k_{sup}$, whenever $\epsilon_2<\epsilon_{sup}$. In words, it is guaranteed  that  running $Alg$   upto the level $k_{sup}$   result in a superset of $\mathcal{I}$ whenever $\epsilon_2<\epsilon_{sup}$. \\
{\bf Assumption 2:-} There exists an  $ \epsilon_{exact}>0$ such that $k_*= k_0$, i.e., $\hat{\mathcal{I}}_{k_0}=Alg({\bf y},{\bf X},k_0)=\mathcal{I}$ whenever $\epsilon_2<\epsilon_{exact}$. In words, it is guaranteed that running $Alg$  precisely $k_0$ iterations, i.e., $Alg_{k_0}$  will recover the true support whenever $\epsilon_2<\epsilon_{exact}$. 

Assumption 1 requires \textit{Alg} to find all the indices from the true support $\mathcal{I}$ within the first $k_{sup}>k_0$ iterations. In other words, Assumption 1 allows  \textit{Alg} to make some errors in the first $k_0$ iterations. Assumption 2 is stronger than Assumption 1 in the sense that  it requires all the first $k_0$ iterations of \textit{Alg}  to be accurate. Assumption 1 guarantees only a superset selection, whereas, Assumption 2 guarantees accurate support recovery for the underlying \textit{Alg}. As we will report later, the RIC conditions imposed on matrix ${\bf X}$ to  satisfy Assumption 2 is much stringent than the conditions to satisfy Assumption 1. 
\begin{thm}\label{thm:tf}
TF-IGP support estimate $\hat{\mathcal{I}}_{\hat{k}_{TF}}$, where $\hat{k}_{TF}=\underset{k}{\arg\min}\ RR(k)$ satisfies  $\mathcal{I}\subseteq \hat{\mathcal{I}}_{\hat{k}_{TF}} $ and $|\hat{\mathcal{I}}_{\hat{k}_{TF}}|\leq k_{sup}$ if algorithm \textit{Alg} satisfy Assumption 1 and $\epsilon_2<\min(\epsilon_{sup},\epsilon_{sig},\epsilon_{\bf X})$. 
Here, $\epsilon_{sig}={\left(\sqrt{1-\delta_{k_{sup}}}\boldsymbol{\beta}_{min}\right)}/{\left(1+\dfrac{\sqrt{1+\delta_{k_{sup}}}}{\sqrt{1-\delta_{k_{sup}}}}\left(2 +  \dfrac{\boldsymbol{\beta}_{max}}{\boldsymbol{\beta}_{min}}\right)\right) }$ is an algorithm independent term depending purely on the signal $\boldsymbol{\beta}$ and  $\epsilon_{\bf X}=\dfrac{\sqrt{1-\delta_{k_{sup}}}\boldsymbol{\beta}_{min}\Gamma_{Alg}({\bf X})}{1+\Gamma_{Alg}({\bf X})}$. $0<\Gamma_{Alg}({\bf X})\leq 1$ in Corollary \ref{corr:beta} is an algorithm dependent term.
\end{thm}
\begin{proof} Please see APPENDIX B.
\end{proof}
Corollary \ref{corr:tf} follows directly from Theorem \ref{thm:tf} by replacing $k_{sup}$ with $k_0$ and Assumption 1  with Assumption 2. 
\begin{corollary}\label{corr:tf} TF-IGP  can recover the true support $ \mathcal{I}$  if algorithm `Alg' satisfy Assumption 2 and $\epsilon_2<\min(\epsilon_{exact},\epsilon_{sig},\epsilon_{\bf X})$. Here,  $\epsilon_{sig}={\left(\sqrt{1-\delta_{k_0}}\boldsymbol{\beta}_{min}\right)}/{\left(1+\dfrac{\sqrt{1+\delta_{k_0}}}{\sqrt{1-\delta_{k_0}}}\left(2 +  \dfrac{\boldsymbol{\beta}_{max}}{\boldsymbol{\beta}_{min}}\right)\right) }$ and $\epsilon_{\bf X}={\sqrt{1-\delta_{k_0}}\boldsymbol{\beta}_{min}\Gamma_{Alg}({\bf X})}/\left({1+\Gamma_{Alg}({\bf X})}\right)$. 
\end{corollary}
\begin{remark}
Note that $\epsilon_{sig}$ decreases with increasing $DR(\boldsymbol{\beta})$. Hence, the SNR required for successful recovery using TF-IGP increases with $DR(\boldsymbol{\beta})$.  This is the main qualitative difference between TF-IGP and the  results for  $\text{OMP}_{k_0}$, $\text{OLS}_{k_0}$ etc. where  $\epsilon_{exact}$ depends only on $\boldsymbol{\beta}_{min}$.  This term can be attributed to the sudden fall in the residual power $\|{\bf r}^{(k)}\|_2$ when a ``very significant" entry in $\boldsymbol{\beta}$ is covered by $Alg({\bf y},{\bf X},k)$ at an intermediate stage $k<k_*$. This mimics the fall in residual power when the  ``last" entry in $\boldsymbol{\beta}$ is selected in the $k_*^{th}$ iteration. 
This later fall in residual power is what TF-IGP trying to detect. The  implication of this result is that the TF-IGP will be lesser effective while recovering $\boldsymbol{\beta}$ with high $DR(\boldsymbol{\beta})$ than in recovering signals with low $DR(\boldsymbol{\beta}) \approx 1$. This observation is also later verified through numerical simulations. Also note that any signal of finite dynamic range $DR(\boldsymbol{\beta})$ can be detected by TF-IGP, if the SNR is made sufficiently high. 
\end{remark}
\subsection{Choice of $k_{max}=\floor{\frac{n+1}{2}}$ in TF-IGP.}
  For successful operation of TF-IGP, i.e. to estimate $k_*$ accurately, it is required that $k_{max}\geq k_*$. In addition to that, it is also required that  $\{{\bf X}_{\hat{\mathcal{I}}_k}\}_{k=1}^{k_{max}}$ are full rank and the residuals $\{{\bf r}^{(k)}\}_{k=0}^{k_{max}}$ are not zero.  It is impossible to ascertain \textit{a priori} when the matrices become rank deficient or residuals become zero. Hence, one can initially set $k_{max}=n$, the maximum value of $k$ beyond which the sub-matrices are rank deficient   and terminate iterations when any of the aforementioned contingencies happen. However, running $n$ iterations of TF-IGP will be computationally demanding. Hence, we have set the value of $k_{max}$ to be $k_{max}=\floor{\frac{n+1}{2}}$. 

{The rationale for this choice of $k_{max}$ is as follows. The Spark of a matrix ${\bf X}$ is defined to be smallest value of $k$ such that $\exists \mathcal{J}\subset[p]$ of cardinality $|\mathcal{J}|=k$ and ${\bf X}_{\mathcal{J}}$ is rank deficient. It is known that   $k_0<\floor{\frac{spark({\bf X})}{2}}$ is  a necessary  condition for sparse recovery using any CS algorithm\cite{elad_book}, even when noise ${\bf w}={\bf 0}_n$. It is also known that $spark({\bf X})\leq n+1$. Hence, for accurate support recovery using any CS algorithm, it is required that $k_0<\floor{\frac{n+1}{2}}$. Thus with a choice of $k_{max}= \floor{\frac{n+1}{2}}$, whenever $Alg({\bf y},{\bf X},k_0)$ returns $\mathcal{I}$, i.e., $k_*=k_0$, $k_{max}=\floor{\frac{n+1}{2}}$ satisfies $k_{max}\geq k_*$ as required by TF-IGP. Note that $k_0<\floor{\frac{n+1}{2}}$ is a necessary condition only for the optimal NP hard $l_0$-minimization. For algorithms like OMP, OLS etc. to deliver exact recovery, i.e., $k_*=k_0$ and $\mathcal{I}=\hat{\mathcal{I}}_{k_0}=Alg({\bf y},{\bf X},k_0)$, $n$ should be as high as  $n=O(k_0^2\log(p))$. Hence, for sparsity levels $k_0$ where OMP/OLS deliver exact support recovery $k_{max}=\floor{\frac{n+1}{2}}$ satisfies $k_{max}\gg k_0$. For OMP, there also exists a situation called extended support recovery\cite{extra} where $k_*$ satisfies $k_*>k_0$ and  $k_*=O(k_0)$, whenever $n=O(k_0\log(p))$.   The choice  of $k_{max}=\floor{\frac{n+1}{2}}$ also leaves sufficient legroom to allow $k_{max}>k_*$ even when $k_*>k_0$ and $k_*=O(k_0)$. }
\begin{remark}
TF-IGP is a  tuning free framework for employing algorithms in the IGP class. The choice of $k_{max}=\floor{\frac{n+1}{2}}$ is signal, noise, matrix and algorithm independent. The user does not have to make any subjective choices in  TF-IGP. 
\end{remark}
\section{Residual ratio threshold   for IGP.}
As revealed by the analysis of TF-IGP, the performance of TF-IGP degrades significantly with increasing $DR(\boldsymbol{\beta})$. TF-IGP is well suited for  applications like  \cite{MLOMP} where $DR(\boldsymbol{\beta})= 1$.  However, there also exist many applications where  $DR(\boldsymbol{\beta})$ is high and TF-IGP framework is highly suboptimal for such applications. This motivates the novel RRT-IGP framework  which can deliver good performance irrespective of $DR(\boldsymbol{\beta})$. RRT-IGP framework is based on the following Theorem.
\begin{thm}\label{thm-rrt}
Let $\Gamma_{Alg}^{lb}({\bf X})\leq \Gamma_{Alg}({\bf X})$ be any lower bound on $\Gamma_{Alg}({\bf X})$ in Corollary 1. The  support estimate $\hat{\mathcal{I}}_{\hat{k}_{RRT}}$, where $\hat{k}_{RRT}=\max\{k:RR(k)<\Gamma_{Alg}^{lb}({\bf X})\}$  satisfies $\mathcal{I}\subseteq \hat{\mathcal{I}}_{\hat{k}_{RRT}}$ and $|\hat{\mathcal{I}}_{\hat{k}_{RRT}}|\leq k_{sup}$ if Assumption 1 is true and $\epsilon_2<\min(\epsilon_{sup},\epsilon_{RRT})$. Here $\epsilon_{RRT}={\Gamma_{Alg}^{lb}({\bf X})\sqrt{1-\delta_{k_{sup}}}\boldsymbol{\beta}_{min}}/\left({1+\Gamma_{Alg}^{lb}({\bf X})}\right)$.
\end{thm}
\begin{proof} Please see APPENDIX D.
\end{proof}
Corollary \ref{corr:rrt}  follows  from Theorem \ref{thm-rrt} by replacing $k_{sup}$ with $k_0$ and Assumption 1 with Assumption 2. 
\begin{corollary}\label{corr:rrt}
The support estimate $\hat{\mathcal{I}}_{\hat{k}_{RRT}}$, where $\hat{k}_{RRT}=\max\{k:RR(k)<\Gamma_{Alg}^{lb}({\bf X})\}$  equals $\mathcal{I}$ if Assumption 2 is true and $\epsilon_2<\min(\epsilon_{exact},\epsilon_{RRT})$. Here $\epsilon_{RRT}={\Gamma_{Alg}^{lb}({\bf X})\sqrt{1-\delta_{k_0}}\boldsymbol{\beta}_{min}}/\left({1+\Gamma_{Alg}^{lb}({\bf X})}\right)$.
\end{corollary}
Thus, if a suitable lower bound or an accurate estimate of $\Gamma_{Alg}({\bf X})$ is available, one can still estimate the support $\mathcal{I}$ or a superset of it  using $\hat{k}_{RRT}=\max\{k:RR(k)<\Gamma_{Alg}^{lb}({\bf X})\}$. This is described in the RRT-IGP algorithm given in TABLE \ref{tab:rrt}.
\begin{table}\centering

\begin{tabular}{|l|}
\hline
{\bf Input:-} Design matrix ${\bf X}$, Observation ${\bf y}$ \\ \ \ \ \ \ \ \ \ \   and lower bound $\Gamma_{Alg}^{lb}({\bf X})\leq \Gamma_{Alg}({\bf X})$.\\ 
{\bf Step 1:-} Run  TF-IGP with ${\bf y}$ as input.\\
{\bf Step 2:-} Compute $\hat{k}_{RRT}=\max\{k:RR(k)<\Gamma_{Alg}^{lb}({\bf X})\}$\\
{\bf Output:-} Support estimate $\hat{\mathcal{I}}_{\hat{k}_{RRT}}$.  \\
\hline
\end{tabular}
\caption{ Residual radio threshold based IGP.}
\label{tab:rrt}
\end{table}
As explained later, one can produce $\Gamma_{Alg}^{lb}({\bf X})$ through procedures that does not require \textit{a priori} knowledge of $k_0$ or $\{\sigma^2,\epsilon_2\}$.  Thus RRT-IGP, just like TF-IGP is also SNO. However,  unlike  $\epsilon_{sig}$ of  TF-IGP, $\epsilon_{RRT}$ of RRT-IGP depends only on $\boldsymbol{\beta}_{min}$ and is independent of $\boldsymbol{\beta}_{max}$. Hence,  RRT-IGP is unaffected by high $DR(\boldsymbol{\beta})$.
\begin{remark}\label{remark:rrt}
The performance of RRT-IGP is  sensitive to the choice of the threshold $\Gamma_{Alg}^{lb}({\bf X})$. When $\Gamma_{Alg}^{lb}({\bf X})$ is very low, then $\epsilon_{RRT}$ will be very low pushing the SNR required for  successful support or superset recovery to higher levels.  Hence, for good performance of RRT-IGP, it is important to produce lower bounds $\Gamma_{Alg}^{lb}({\bf X})$ closer to $\Gamma_{Alg}({\bf X})$.
\end{remark}
\subsection{Selection of $\Gamma_{Alg}^{lb}({\bf X})$ in RRT-IGP.} 
{As mentioned in Remark \ref{remark:rrt}, the choice of $\Gamma_{Alg}^{lb}({\bf X})$ is crucial to the performance of RRT-IGP. $\Gamma_{Alg}^{lb}({\bf X})$ can be either a lower bound  dependent on the given matrix ${\bf X}$   or it can be an universal lower bound (i.e., independent of ${\bf X}$). Universal lower bounds are more useful, because  they does not require any extra computations involving the particular matrix ${\bf X}$. A ready made universal lower bound  is $\Gamma_{RRT}^{\alpha}$  in Theorem \ref{thm:Beta} which for small values of $\alpha$ like $\alpha=0.01$  will be lower than $\Gamma_{Alg}({\bf X})$ with a very high probability. However, $\Gamma_{RRT}^{\alpha}$  involves two levels of union bounds and hence is a pessimistic bound in the sense that   $\Gamma_{RRT}^{\alpha}$ tends to be much lower than $\Gamma_{Alg}({\bf X})$. Note that a lower value of $\Gamma_{Alg}^{lb}({\bf X})$ results in an increase in the  SNR required for successful recovery. Further, $\Gamma_{RRT}^{\alpha}$ does not capture the properties of the particular `Alg'.  Next we explain a  numerical method to produce universal lower bounds on $\Gamma_{Alg}({\bf X})$ that delivered better empirical NMSE performance than $\Gamma_{RRT}^{\alpha}$.

The reproducibility property A2) of IGP  implies that the index selected in the $k^{th}$ iteration depends on the previous iterations only through  the residual ${\bf r}^{(k-1)}$ in the $k-1^{th}$ iteration. This means that ${\bf r}^{(k)}$ and hence $RR(k)$ for $k>k_*$   can be recreated by running IGP with ${\bf r}^{(k_*)}=({\bf I}_n-{\bf P}_{\hat{\mathcal{I}}_{k_*}}){\bf w}$ as input. That is, $RR(k)$ for $k>k_*$ and hence $\Gamma_{Alg}({\bf X})$ depends only on how \textit{Alg} update its indices when provided with a  noise only vector ${\bf r}^{(k_*)}=({\bf I}_n-{\bf P}_{\hat{\mathcal{I}}_{k_*}}){\bf w}$ as input. This observation motivates the noise assisted training scheme for IGP given in TABLE \ref{tab:noise} where we try to produce a universal lower bound $\Gamma_{Alg}^{lb-tr}$ on $\Gamma_{Alg}({\bf X})$ by training the particular IGP `Alg' multiple times with independently generated noise samples ${\bf y}^s\sim \mathcal{N}({\bf 0}_n,{\bf I}_n)$ and matrices ${\bf X}^s$ of the same dimensions. Unlike $\Gamma_{RRT}^{\alpha}$, the offline training in TABLE \ref{tab:noise} exploits the properties of the particular IGP `Alg' which explains its better empirical performance. The training in TABLE \ref{tab:noise} need to be done only once for a given value of $n$ and $p$. This process is completely independent of the given matrix ${\bf X}$.  Hence,  $\Gamma_{Alg}^{lb-tr}$ can be computed completely offline. }

\begin{table}\centering

\begin{tabular}{|l|}
\hline
{\bf Input:-} Matrix dimensions $n$ and $p$. Number of training symbols $N_{tr}$.\\ 

\ \ \ \ \ \ \ \ \ \ Repeat Steps 1-3 for $s=1$ to $N_{tr}$. \\
{\bf Step 1:-} Generate ${\bf y}^s\overset{i.i.d}{\sim} \mathcal{N}({\bf 0}_n,{\bf I}_n)$ and and ${\bf X}_{i,j}^s\overset{\textit{i.i.d}}{\sim}\mathcal{N}(0,1)$.\\
{\bf Step 2:-} Run IGP as outlined in TABLE \ref{tab:tf-omp} with (${\bf y}^s,{\bf X}^s$) as input.\\
{\bf Step 3:-} Compute $RR_{min}^s=\underset{1\leq k\leq k_{max}}{\min} RR(k)$. \\

{\bf Output:-}  $\Gamma_{Alg}^{lb-tr}=\underset{1\leq s\leq N_{tr}}{\min}RR_{min}^s.$ \\
\hline
\end{tabular}
\caption{ Noise Assisted Offline Training for IGP.}
\label{tab:noise}
\end{table}
\begin{remark}
{The performance of RRT-IGP depends  on the number of  training samples $N_{tr}$. Hence, RRT-IGP with $\Gamma_{Alg}^{lb-tr}$ is SNO, but not tuning free. Since the training is completely offline, one can set $N_{tr}$ to arbitrarily high values pushing $\Gamma_{Alg}^{lb-tr}$ to be lower than $\Gamma_{Alg}({\bf X})$ with a very high probability. We have observed that the estimated  $\Gamma_{Alg}^{lb-tr}$ and the resultant NMSE performance with its usage in RRT-IGP framework is largely invariant to $N_{tr}$ as long as $N_{tr}$ is of the order of hundreds. Hence, with a large value of $N_{tr}$, RRT-IGP is only very weakly dependent on $N_{tr}$.}
\end{remark}

\subsection{Effect of $n$, $p$ and $N_{tr}$  on $\Gamma_{Alg}^{lb-tr}$,  $\alpha$ on $\Gamma_{RRT}^{\alpha}$.}
\begin{figure}[htb]
\includegraphics[width=\columnwidth]{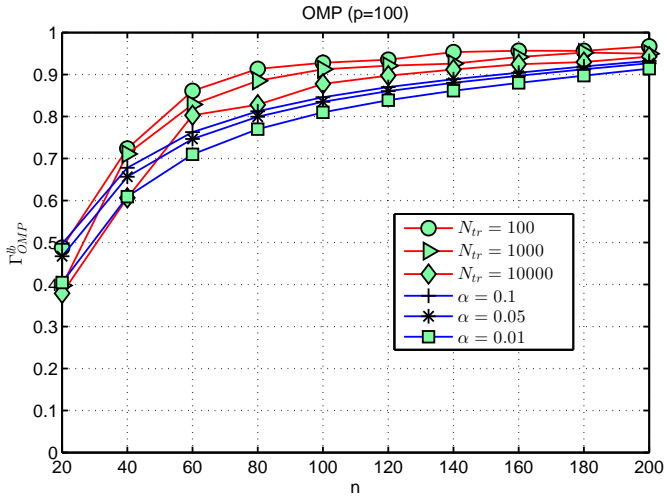}
\caption{Effect of n, p and $N_{tr}$ on $\Gamma_{Alg}^{lb-tr}$ and $\alpha$ on $\Gamma_{RRT}^{\alpha}$. }
\label{fig.lower_bounds}
\end{figure}
{In Fig.\ref{fig.lower_bounds}, we plot the effect of dimensions $n$, $p$ and number of training samples $N_{tr}$ on  $\Gamma_{Alg}^{lb-tr}$ generated by the noise assisted training scheme in TABLE \ref{tab:noise} for OMP. It can be observed from Fig.\ref{fig.lower_bounds} that   $\Gamma_{OMP}^{lb-tr}$ increases with increasing $n/p$ ratio. Similar trends were visible in a large number of numerical  simulations. It is also clear from Fig.\ref{fig.lower_bounds} that  $\Gamma_{OMP}^{lb-tr}$ does not vary much with the number of training samples $N_{tr}$ except when $n/p$ is too low (like $n/p=0.2$). {Fig.\ref{fig.lower_bounds} also demonstrates that $\Gamma^{\alpha}_{RRT}$ in Theorem 1 with  values of $\alpha$ like $\alpha=0.1$ or  $\alpha=0.01$  will be lower than $\Gamma_{Alg}^{lb-tr}$. This explains why $\Gamma_{Alg}^{lb-tr}$ is a better candidate for $\Gamma^{lb}_{Alg}({\bf X})$ in the RRT-IGP framework than $\Gamma^{\alpha}_{RRT}$.} }
\subsection{Computational complexity of TF-OMP and RRT-OMP.}
 All the quantities required for TF-OMP/RRT-OMP can be obtained from a single  computation of OMP  with sparsity level $k_{max}=\floor{\frac{n+1}{2}}$. Since computing OMP with sparsity level $k$ has complexity $O(knp)$, the online complexity of TF-OMP/RRT-OMP is $O(n^2p)$.   When $k_0\ll n$, the complexity of TF-OMP/RRT-OMP is nearly $n/(2k_0)$ times higher than the $O(k_0np)$ complexity of $\text{OMP}_{k_0}$. However, unlike $\text{OMP}_{k_0}$, TF-OMP/RRT-OMP does not require \textit{a priori} knowledge of $k_0$. The complexity of PaTh(OMP), i.e., OMP applied in the  SNO PaTh framework is $O(k_0np)$. However, TF-OMP/RRT-OMP significantly outperforms PaTh(OMP) in many situations.  The  tuning free  SPICE algorithm is solved using an iterative scheme where each iteration has complexity  $O\left((n+p)^3\right)$\cite{spice_like}.  The complexity of TF-OMP/RRT-OMP is significantly lower than the complexity of SPICE. Similar complexity comparisons hold true for TF-OLS/RRT-OLS  also.
\section{Comparison of TF-OMP/RRT-OMP with $\text{OMP}_{k_0}$.}
We next compare Theorems 2-3 and Corollaries 2-3 in the light of existing literature on support recovery using $\text{OMP}_{k_0}$ and $\text{OMP}_{\epsilon_2}$.  Similar conclusions hold true for OLS in the light of support recovery conditions in \cite{ERC-OMP,OLSarxiv}.
\subsection{Exact support recovery using $\text{OMP}_{k_0}$ and $\text{OMP}_{\epsilon_2}$ .}
{ The best known conditions for successful support recovery using $\text{OMP}_{k_0}$ and $\text{OMP}_{\epsilon_2}$ in bounded noise $\|{\bf w}\|_2\leq \epsilon_2$ is summarized below. Please refer to \cite{omp_necess} for details.
\begin{lemma}\label{lemma:success}
 $\text{OMP}_{k_0}$ or $\text{OMP}_{\epsilon_2}$  successfully recover the support $\mathcal{I}$ if $\delta_{k_0+1}<{1}/{\sqrt{k_0+1}}$ and $\epsilon_2\leq \epsilon_{exact}=
 \boldsymbol{\beta}_{min}\sqrt{1-\delta_{k_0+1}}\left[1+\dfrac{\sqrt{1-\delta_{k_0+1}^2}}{1-\sqrt{k_0+1}\delta_{k_0+1}}\right]^{-1}.$ 
\end{lemma}
{In words, if $\delta_{k_0+1}$ is sufficiently low, then $\text{OMP}_{k_0}$ or $\text{OMP}_{\epsilon_2}$ will recover the true support at high SNR.  This condition is also worst case necessary in the following sense.  There exist \textit{some matrix $\tilde{\bf X}$} with $\delta_{k_0+1}\geq {1}{\sqrt{k_0+1}}$ for which there exist \textit{a $k_0$ sparse signal} $\tilde{\boldsymbol{\beta}}$ that cannot be recovered using $\text{OMP}_{k_0}$. Note that $\delta_{k_0+1}<1/\sqrt{k_0+1}$  is not a necessary condition for the given matrix ${\bf X}$.   We next compare and contrast the sufficient conditions for exact support recovery in Corollaries 2-3 (reproduced in TABLE \ref{tab:guarentee}) with the result in Lemma \ref{lemma:success}. }

\begin{table}[t]\centering

\begin{tabular}{|c|c|}
\hline
TF-OMP&RRT-OMP \\
 $\epsilon_2\leq \min(\epsilon_{exact},\epsilon_{sig},\epsilon_{\bf X})$ &$\epsilon_{2}\leq \min(\epsilon_{exact},\epsilon_{RRT})$   \\ \hline
$\epsilon_{X}=\dfrac{\Gamma_{Alg}({\bf X})\sqrt{1-\delta_{k_0}}\boldsymbol{\beta}_{min}}{1+\Gamma_{Alg}({\bf X})}$  & $\epsilon_{RRT}=\sqrt{1-\delta_{k_0}}\boldsymbol{\beta}_{min}$ \\ 
 $\epsilon_{sig}=\dfrac{\sqrt{1-\delta_{k_0}}\boldsymbol{\beta}_{min}}{\left(1+\dfrac{\sqrt{1+\delta_{k_0}}}{\sqrt{1-\delta_{k_0}}}\left(2 +  \dfrac{\boldsymbol{\beta}_{max}}{\boldsymbol{\beta}_{min}}\right)\right) }$ & $\times \dfrac{\Gamma_{Alg}^{lb}({\bf X})}{1+\Gamma_{Alg}^{lb}({\bf X})}$ \\
\hline
\end{tabular}
\caption{Support recovery conditions: TF-OMP/RRT-OMP. }
\label{tab:guarentee}
\end{table}
First consider the case of RRT-OMP. Note that the quantity $SNR_{excess}^{RRT}=\dfrac{\epsilon_{exact}}{\epsilon_{RRT}}$ is a measure of excess SNR required for RRT-OMP to ensure successful support recovery in comparison with $\text{OMP}_{k_0}$. 
 Substituting the values of $\epsilon_{exact}$ and $\epsilon_{RRT}$ and using the bound $\delta_{k_0}\leq\delta_{k_0+1}$ in a) of Lemma 1 gives 
\begin{equation}
SNR_{excess}^{RRT}\leq \dfrac{1+\frac{1}{\Gamma_{Alg}^{lb}({\bf X})}}{1+\frac{\sqrt{1-\delta_{k_0+1}^2}}{1-\sqrt{k_0+1}\delta_{k_0+1}} }.
\end{equation} 
Note that $\frac{\sqrt{1-\delta_{k_0+1}^2}}{1-\sqrt{k_0+1}\delta_{k_0+1}}=\left(\frac{1-\delta_{k_0+1}}{1-\sqrt{k_0+1}\delta_{k_0+1}}\right)\sqrt{\frac{1+\delta_{k_0+1}}{1-\delta_{k_0+1}}}\geq 1$. Consequently, $SNR_{excess}^{RRT}\leq 0.5 \left(1+\frac{1}{\Gamma_{Alg}^{lb}({\bf X})}\right)$.   For $\Gamma_{Alg}^{lb}({\bf X})=0.4$, $SNR_{excess}^{RRT}\leq 1.75$ and for $\Gamma_{Alg}^{lb}({\bf X})=0.8$, $SNR_{excess}^{RRT}\leq 1.125$. {From Fig.\ref{fig.lower_bounds}, one can see that lower bounds $\Gamma_{Alg}^{lb-tr}$ and $\Gamma^{\alpha}_{RRT}$ on $\Gamma_{Alg}({\bf X})$ used for implementing RRT-OMP  increases fast with increasing $n/p$ (in Fig.\ref{fig.lower_bounds} $\Gamma_{Alg}^{lb-tr}\geq 0.8$ as long as $n/p>0.5$).} Hence, the excess SNR required for RRT-OMP to accomplish successful support recovery is negligible as long as $n/p$ is moderately high. 

Next we consider the terms $\epsilon_{sig}$ and $\epsilon_{\bf X}$ in TF-OMP. By definition, $\Gamma_{Alg}({\bf X})\geq \Gamma_{Alg}^{lb}({\bf X})$ and hence $SNR_{excess}^{\bf X}=\dfrac{\epsilon_{exact}}{\epsilon_{\bf X}}$ follows exactly that of $SNR_{excess}^{RRT}$.     Similar to $SNR_{excess}^{RRT}$, one can bound $SNR_{excess}^{sig}=\dfrac{\epsilon_{exact}}{\epsilon_{sig}}$ as 
\begin{equation}
SNR_{excess}^{sig}\leq 0.5 {\left(1+\frac{\sqrt{1+\delta_{k_0}}}{\sqrt{1-\delta_{k_0}}}\left(2 +  \frac{\boldsymbol{\beta}_{max}}{\boldsymbol{\beta}_{min}}\right)\right) }
\end{equation}
{Note that exact recovery is possible only if $\delta_{k_0}\leq\delta_{k_0+1}\leq 1/\sqrt{k_0+1}$. Hence, for exact recovery,  $\delta_{k_0}$ should be very low. Consequently, one can approximate $\delta_{k_0}\approx 0$ and hence, $SNR_{excess}^{sig}\leq 0.5 (3+\frac{\boldsymbol{\beta}_{max}}{\boldsymbol{\beta}_{min}})$.
For uniform signals, i.e., $\frac{\boldsymbol{\beta}_{max}}{\boldsymbol{\beta}_{min}}\approx 1$, $SNR_{excess}^{sig}\leq 2$.  However, $SNR_{excess}^{sig}$  increases tremendously with the increase in $DR(\boldsymbol{\beta})$. To summarize,  the performance of RRT-GP $(\forall\boldsymbol{\beta})$ and TF-IGP ($\boldsymbol{\beta}$ with low $DR(\boldsymbol{\beta}))$  compares very favourably with $\text{OMP}_{k_0}$ or $\text{OMP}_{\epsilon_2}$ in terms of the SNR required for exact recovery. } }
\subsection{High SNR consistency (HSC) of TF-OMP.}
HSC of variable selection techniques  in Gaussian noise, i.e., ${\bf w} \sim \mathcal{N}({\bf 0}_,\sigma^2{\bf I}_n)$ has received considerable attention in signal processing community \cite{tsp,cs_tsp}. HSC is  defined as follows. \\
{\bf Definition 4:-} A support estimate $\hat{\mathcal{I}}$ of $\mathcal{I}=supp(\boldsymbol{\beta})$ is high SNR consistent iff
 $PE=\mathbb{P}(\hat{\mathcal{I}}\neq \mathcal{I})$ satisfies $\underset{\sigma^2 \rightarrow 0 }{\lim}PE=0$.  \\
The necessary and sufficient conditions for the HSC of LASSO and OMP  are  derived in  \cite{cs_tsp}. OMP with SC  $\|{\bf r}^{(k)}\|_2\leq \gamma$ or $\|{\bf X}^T{\bf r}^{(k)}\|_{\infty}\leq \gamma$ are high SNR consistent if $\underset{\sigma^2 \rightarrow 0}{\lim}\gamma=0$ and $\underset{\sigma^2 \rightarrow 0}{\lim}\dfrac{\gamma}{\sigma}=\infty$. It was also shown that $\text{OMP}_{k_0}$ is  high SNR consistent, whereas, $\text{OMP}_{\sigma^2}$ is inconsistent at high SNR. These results are useful only if  \textit{a priori} knowledge of $k_0$ or $\sigma^2$ are available.  We next establish the HSC of TF-OMP. This is a first time a SNO CS algorithm is reported  to achieve HSC. 
 \begin{thm}TF-OMP is high SNR consistent for any  $\boldsymbol{\beta}$ with $|supp(\boldsymbol{\beta})|\leq k_0$ whenever  $\delta_{k_0+1}< {1}/{\sqrt{k_0+1}}$. 
 \end{thm}
 \begin{proof}
  TF-OMP recover the correct support whenever $\|{\bf w}\|_2< \epsilon_{TF}=\min(\epsilon_{exact},\epsilon_{sig},\epsilon_{\bf X})$ and   $\epsilon_{TF}>0$ is constant strictly bounded away from zero (Corollary \ref{corr:tf} and Lemma \ref{lemma:success}). Hence, $\mathbb{P}(\hat{\mathcal{I}}=\mathcal{I})$ satisfies $\mathbb{P}(\hat{\mathcal{I}}=\mathcal{I})\geq \mathbb{P}(\|{\bf w}\|_2^2\leq \epsilon_{TF}^2)=\mathbb{P}(\dfrac{\|{\bf w}\|_2^2}{\sigma^2}\leq \dfrac{\epsilon_{TF}^2}{\sigma^2}) $. Note that $T=\dfrac{\|{\bf w}\|_2^2}{\sigma^2}\sim \chi^2_n$ and $T$ is a bounded in probability R.V with distribution  independent of $\sigma^2$.  Hence,  $\underset{\sigma^2 \rightarrow 0}{\lim}\mathbb{P}(\hat{\mathcal{I}}= \mathcal{I})\geq \underset{\sigma^2 \rightarrow 0}{\lim}\mathbb{P}(T<\dfrac{\epsilon_{TF}^2}{\sigma^2})=1$.
\end{proof}
\begin{remark}\label{PE_floor}{Once the lower bounds $\Gamma^{lb}_{\text{OMP}}({\bf X})$ in RRT-OMP satisfy  $\Gamma^{lb}_{\text{OMP}}({\bf X})\leq \Gamma_{\text{OMP}}({\bf X})$ almost surely, then RRT-OMP is also high SNR consistent. However, both the offline training scheme in TABLE \ref{tab:rrt} with high $N_{tr}$ or $\Gamma^{\alpha}_{RRT}$ in Theorem \ref{thm:Beta} with very small $\alpha$ guarantees $\Gamma^{lb-tr}_{\text{OMP}}\leq \Gamma_{\text{OMP}}({\bf X})$ or $\Gamma^{\alpha}_{RRT}\leq \Gamma_{\text{OMP}}({\bf X})$ with a very high probability, not almost surely. Hence, RRT-OMP with  $\Gamma^{lb}_{\text{OMP}}({\bf X})$ produced using these schemes are not guaranteed to be high SNR consistent. Numerical simulations in Section \rom{7} indicates that PE performance of RRT-OMP at high SNR is  better than that of $\text{OMP}_{\sigma^2}$.}
\end{remark}
\subsection{Impact of extended recovery in TF-OMP/RRT-OMP.} 
Both TF-IGP and RRT-IGP frameworks try to estimate $k_*=\min\{k:\hat{\mathcal{I}}_k\supseteq \mathcal{I}\}$, i.e., the smallest superset generated by the IGP solution path.  So far we have considered the case when $k_*=k_0$ in which case TF-OMP/RRT-OMP try to estimate  $\mathcal{I}$ directly. We next  consider the case  when $k_*>k_0$, a situation referred to as extended recovery\cite{extra,prateek} in literature.  Lemma \ref{lemma:extended} summarizes the latest results on extended recovery for  OMP. 
\begin{lemma}\label{lemma:extended}
  OMP can recover any $k_0$ sparse signal in $2k_0$ iterations whenever $\delta_{4k_0}\leq 0.2 $ or in $3k_0$ iterations whenever $\delta_{5k_0}\leq 0.33$\cite{prateek}.
 \end{lemma}
The requirement $\delta_{4k_0}\leq 0.2$ for extended recovery is much weaker than the condition $\delta_{k_0+1}\leq 1/{\sqrt{k_0+1}}$ required for exact recovery\cite{extra}. There is also a qualitative difference between these two conditions. For a random matrix ${\bf X}_{i,j}\overset{i.i.d}{\sim}\mathcal{N}(0,1)$, $\delta_{ck}<a$ will hold true with a high probability whenever $n=O\left(\frac{ck}{a^2}\log(\frac{p}{k})\right)$. Hence, for $\delta_{4k_0}<0.2$, one need only $n=O\left(k_0\log(p)\right)$ measurements, whereas, for exact recovery, i.e., $\delta_{k_0+1}<1/\sqrt{k_0+1}$, one need a significantly higher $n=O\left(k_0^2\log(p)\right)$ measurements. Hence, for a fixed $n$, the range of $k_0$ that allow for extended recovery is much higher than that of exact recovery. Without loss of generality,  we focus on the condition $\delta_{4k_0}\leq 0.2$ which ensures that $k_*\leq 2k_0$.   Extended recovery results are available only for noiseless  data.  However, one can conjecture that these results hold true for noisy case also as long as SNR is sufficiently high, i.e., $\exists \epsilon_{sup}>0$ such that  $\epsilon_2<\epsilon_{sup}$ implies $k_*\leq 2k_0$.

 Consider a sparsity level $k_0$ where exact support recovery condition for OMP, i.e.,  $\delta_{k_0+1}<1/\sqrt{k_0+1}$ does not hold and extended recovery condition $\delta_{4k_0}<0.2$ hold true. From the difference in the scaling rules of both these conditions, i.e.,  $n=O(k_0^2\log(p))$ and $n=O\left(k_0\log(p)\right)$, it is true that for many matrices such sparsity regimes exist.  Also assume that the support $\mathcal{I}$ with $|\mathcal{I}|=k_0$ is sampled uniformly from $[p]$. {This sparsity regime implies that for many signals $\boldsymbol{\beta}\in \mathcal{B}_1$, the support estimate returned by $\text{OMP}_{k_0}$, i.e., $\hat{\mathcal{I}}_{k_0}=\text{OMP}({\bf y},{\bf X},k_0)$  misses some indices, i.e., $|\mathcal{I}/\hat{\mathcal{I}}_{k_0}|>0$, whereas, for some signals $\boldsymbol{\beta}\in \mathcal{B}_2$, $\text{OMP}_{k_0}$ returns the correct estimate.}  For  signals where $\text{OMP}_{k_0}$ gives erroneous output, i.e., $\boldsymbol{\beta}\in \mathcal{B}_1$, the following bound is proved  in APPENDIX E.
\begin{equation}\label{mismatch1}
\|\boldsymbol{\beta}-\hat{\boldsymbol{\beta}}\|_2{\geq}(1-\frac{\delta_{2k_0}}{1-\delta_{k_0}})\boldsymbol{\beta}_{min}-\frac{\epsilon_2}{\sqrt{1-\delta_{k_0}}}, \forall \epsilon_2>0.
\end{equation}
For the same signal $\boldsymbol{\beta}\in \mathcal{B}_1$,  consider the case with RRT-OMP when $\epsilon_2\leq \min(\epsilon_{sup},\epsilon_{RRT})$. At this SNR level, $\delta_{4k_0}\leq 0.2$ implies that  $k_*\leq 2k_0$ and RRT-OMP detects $k_*$ accurately, i.e., $\hat{k}_{RRT}=k_*$ and the support estimate $\hat{\mathcal{I}}_{\hat{k}_{RRT}}$ satisfies $|\hat{\mathcal{I}}_{\hat{k}_{RRT}}|\leq 2k_0$ and $\mathcal{I}\subset \hat{\mathcal{I}}_{\hat{k}_{RRT}}$ (Theorem \ref{thm-rrt}). From a)  and c) of Lemma 1 and $\hat{k}_{RRT}=k_*\leq 2k_0$, we have
\begin{equation}\label{match}
\begin{array}{ll}
\|\boldsymbol{\beta}-\hat{\boldsymbol{\beta}}\|_2&=\|\boldsymbol{\beta}_{\hat{\mathcal{I}}_{k_*}}-{\bf X}_{\hat{\mathcal{I}}_{k_*}}^{\dagger}({\bf X}_{\hat{\mathcal{I}}_{k_*}}\boldsymbol{\beta}_{\hat{\mathcal{I}}_{k_*}}+{\bf w})\|_2 \\
&=\|{\bf X}_{\hat{\mathcal{I}}_{k_*}}^{\dagger}{\bf w}\|_2\leq  \dfrac{\epsilon_2}{\sqrt{1-\delta_{2k_0}}}, 
\end{array}
\end{equation} 
$\forall \epsilon_2\leq \min(\epsilon_{sup},\epsilon_{RRT})$. As SNR increases, the error in the $\text{OMP}_{k_0}$ estimate for signals $\boldsymbol{\beta}\in \mathcal{B}_1$ satisfies $\|\boldsymbol{\beta}-\hat{\boldsymbol{\beta}}\|_2\gtrsim (1-\frac{\delta_{2k_0}}{1-\delta_{k_0}})\boldsymbol{\beta}_{min}$, whereas, the error in RRT-OMP estimate (\ref{match}) converges to zero. Next consider signals $\boldsymbol{\beta}\in \mathcal{B}_2$ for which  $\text{OMP}_{k_0}$ returns correct support, i.e., $k_*=k_0$. For $\boldsymbol{\beta}\in \mathcal{B}_2$, RRT-OMP also identify the true support correctly, whenever $\epsilon_2\leq \min(\epsilon_{sup},\epsilon_{RRT})$. Following (\ref{match}), the error  $\|\boldsymbol{\beta}-\hat{\boldsymbol{\beta}}\|_2$ for $\boldsymbol{\beta}\in \mathcal{B}_2$  at high SNR for both $\text{OMP}_{k_0}$ and RRT-OMP satisfies $\|\boldsymbol{\beta}-\hat{\boldsymbol{\beta}}\|_2\leq\frac{\epsilon_2}{\sqrt{1-\delta_{k_0}}}$, $\forall \epsilon_2\leq \min(\epsilon_{sup},\epsilon_{RRT})$. This error converges to zero as SNR increases.   Hence, when the supports are randomly sampled, $\text{OMP}_{k_0}$ suffers from error floors at high SNR due to the irrecoverable signals $\boldsymbol{\beta}\in \mathcal{B}_1$, whereas, RRT-OMP recover all signals and does not suffer from error floors.  This analysis explains why RRT-OMP outperform $\text{OMP}_{k_0}$ in certain SNR sparsity regimes (Figures 4 and 6 in Section \rom{7}.)
\begin{remark} Unlike $\text{OMP}_{k_0}$ where the number of iterations are fixed \textit{a priori}, iterations in $\text{OMP}_{\epsilon_2}$ or $\text{OMP}_{\sigma^2}$  does not stop until $\|{\bf r}^{(k)}\|_2\leq\epsilon_2$ or $\|{\bf r}^{(k)}\|_2\leq\sigma\sqrt{n+2\sqrt{n\log(n)}}$. Hence, the iterations in $\text{OMP}_{\epsilon_2}$ and $\text{OMP}_{\sigma^2}$ can go beyond $k_0$  until all the entries in $\boldsymbol{\beta}$ are selected, i.e.,  $\text{OMP}_{\epsilon_2}$ and $\text{OMP}_{\sigma^2}$ can automatically adjust to extended recovery. This explains why $\text{OMP}_{\sigma^2}$ and  performs better than $\text{OMP}_{k_0}$ in  Section \rom{7}.
\end{remark}
\section{Numerical Simulations}
In this section, we numerically evaluate the performance of TF-IGP/RRT-IGP  and provide insights into the  strengths and shortcomings of the same. Due to space constraints, simulation results are provided only for variants of OMP. Exactly similar inferences can be made for OLS too.  For satisfactory  asymptotic performance,  the user defined parameter $c$ in PaTh(OMP)  has to satisfy $c>1$.  For finite dimensional problems, a choice of $0.5<c<1.5$ is recommended\cite{vats2014path}. Hence, we set the parameter $c=1.1$. RRT-OMP uses $\Gamma^{lb-tr}_{Alg}$ in TABLE \ref{tab:rrt} with $N_{tr}=1000$.   All results in Fig.3-7  are computed after performing $10^4$ iterations. 
\subsection{Matrix and Signal Models.} 
We considered two matrix models in our simulations. One is the usual Gaussian random matrix with \textit{i.i.d} $\mathcal{N}(0,1)$ and $l_2$ normalized columns. These matrices are independently generated in each iteration. The second matrix we consider is the structured matrix formed by the concatenation of two orthonormal matrices,  ${\bf I}_n$ and $n\times n$ Hadamard matrix ${\bf H}_n$, i.e., ${\bf X}=[{\bf I}_n,{\bf H}_n]$. ${\bf X}=[{\bf I}_n,{\bf H}_n]$ guarantees exact recovery of all signals by $\text{OMP}_{k_0}$ with $k_0\leq 1/\sqrt{n}$\cite{elad_book}.
We consider two signal models for simulations. One is the uniform signal model where all non zero entries  of $\boldsymbol{\beta}$ are sampled randomly from $\{1,-1\}$.  Second   is the random signal model where the non zero entries are  sampled independently from a $\mathcal{N}(0,1)$ distribution, i.e., $\boldsymbol{\beta}_j\overset{i.i.d}{\sim}\mathcal{N}(0,1)$ for $j \in \mathcal{I}$. Random signal model exhibits a very high $DR(\boldsymbol{\beta})$, whereas, the uniform signal model exhibits $DR(\boldsymbol{\beta})=1$. \squeezeup
\subsection{ Small sample performance. }
In this section, we evaluate the performance of algorithms when $n,p$ and $k_0$ are small.  For the uniform signal model in Fig.\ref{fig.Had32_MSE}, the performance of TF-OMP and RRT-OMP matches the performance of $\text{OMP}_{\sigma^2}$ throughout the SNR range and $\text{OMP}_{k_0}$ right from SNR=10dB.  PaTh(OMP) suffers from severe error floors. For the random signal model, the performance of TF-OMP deteriorates significantly because of the high $DR(\boldsymbol{\beta})$, whereas, RRT-OMP performs very close to that of $\text{OMP}_{\sigma^2}$ and $\text{OMP}_{k_0}$. In both signal models, RRT-OMP performs  significantly better than PaTh(OMP) which is also SNO. Unlike the  uniform signal model, the performance of PaTh(OMP) is much better in random signal model. A similar trend is visible in  Fig.\ref{fig.random32_MSE} except that $\text{OMP}_{k_0}$ exhibit NMSE floors when the signal is uniform and   TF-OMP/RRT-OMP outperforms $\text{OMP}_{k_0}$  and matches the performance of $\text{OMP}_{\sigma^2}$. These results can be explained by the reasoning given in Section \rom{6}.C. 

{In Fig.\ref{fig.PE}, we present the support recovery performance of algorithms. From the left side of Fig.\ref{fig.PE}, one can see that the PE of $\text{OMP}_{k_0}$ and TF-OMP decreases with increasing SNR, whereas, $\text{OMP}_{\sigma^2}$ and PaTh(OMP) suffer error floors. RRT-OMP also exhibits PE flooring at high SNR at a PE level much lower than that of $\text{OMP}_{\sigma^2}$. This is explained in Remark \ref{PE_floor}.  Note that ${\bf X}=[{\bf I}_{32},{\bf H}_{32}]$ guarantees exact support recovery whenever $k_0\leq 3$. However, there exist a non zero probability that a $32\times 64$  random matrix fails to satisfy the RIC condition required for exact recovery. This explains the  PE floors at high SNR  in the R.H.S of Fig.\ref{fig.PE}. Nevertheless, the PE of RRT-OMP and TF-OMP matches that of $\text{OMP}_{k_0}$ and $\text{OMP}_{\sigma^2}$.}
\begin{figure}[htb]
\includegraphics[width=\columnwidth]{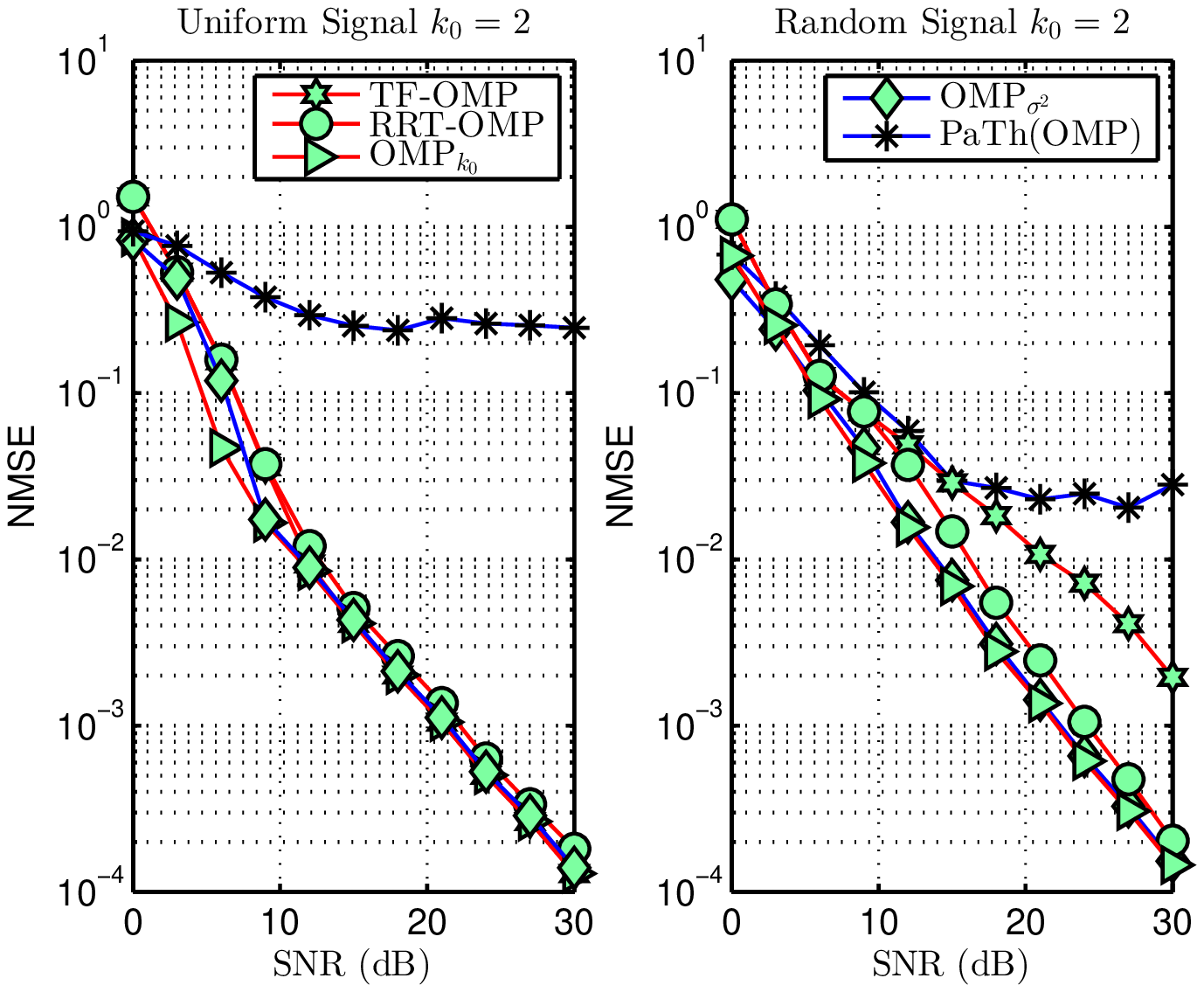}
\caption{NMSE performance: ${\bf X}=[{\bf I}_{16},{\bf H}_{16}]$ and $k_0=2$. }
\label{fig.Had32_MSE}
\end{figure}
\begin{figure}[htb]
\includegraphics[width=\columnwidth]{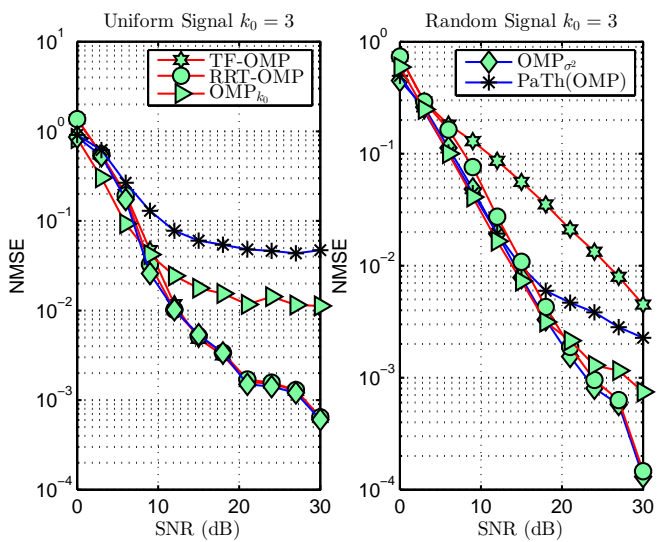}
\caption{NMSE performance: ${\bf X}_{32\times 64}$ is random and $k_0=3$. }
\label{fig.random32_MSE}
\end{figure}
\begin{figure}
\includegraphics[width=\columnwidth]{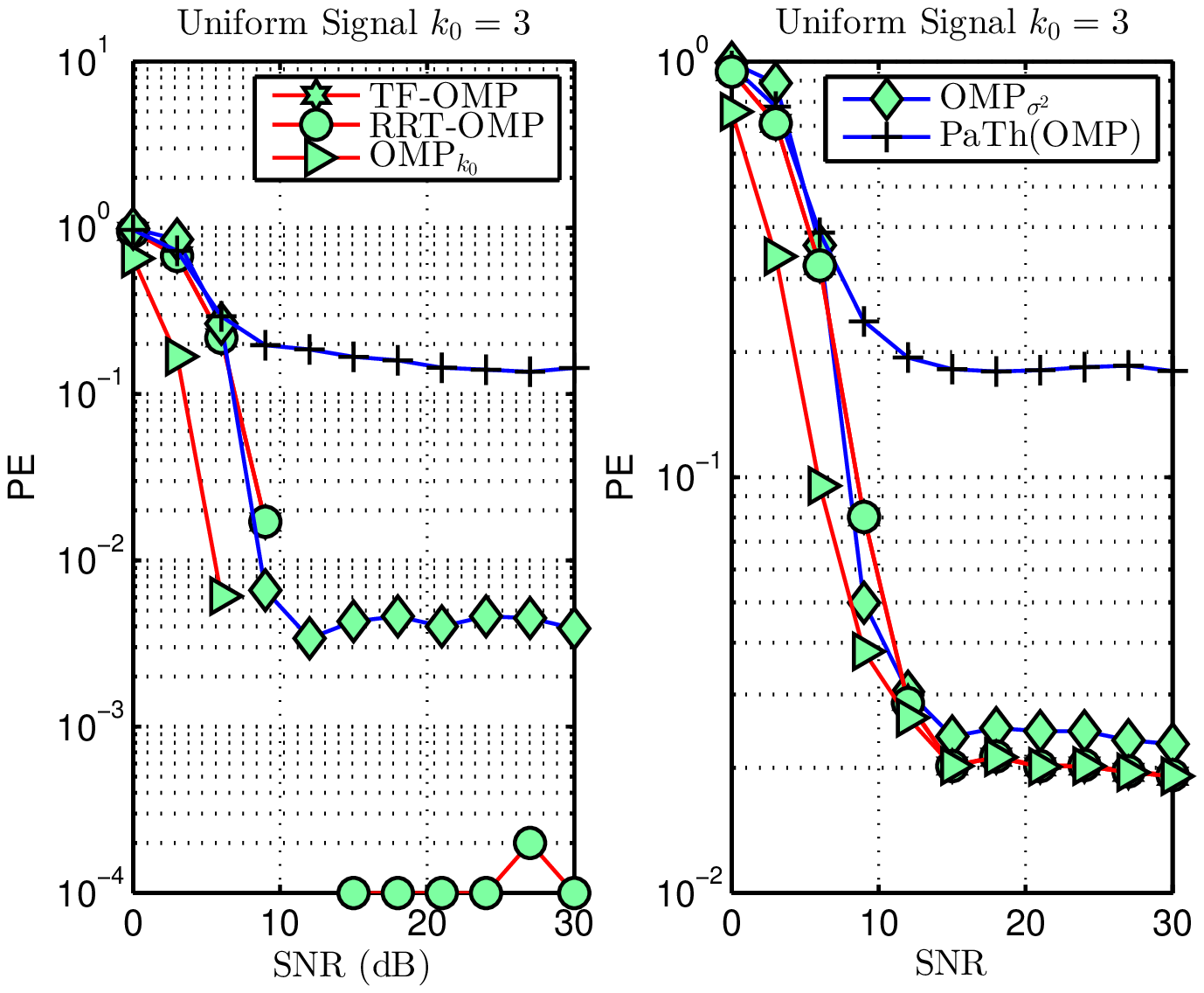}
\caption{PE: (left) ${\bf X}=[{\bf I}_{32},{\bf H}_{32}]$. (right)  ${\bf X}_{32 \times 64}$ random. }
\label{fig.PE}
\end{figure}
\subsection{ Large sample performance.}
Next we consider the performance of algorithms in Fig.\ref{fig.n200_unif} when $n$ and $p$ are high and signal is uniform.  From Fig.\ref{fig.n200_unif}, one can see that at low SNR (SNR=10dB), both TF-OMP and RRT-OMP outperforms $\text{OMP}_{\sigma^2}$ and  $\text{OMP}_{k_0}$, whereas at high SNR,  TF-OMP, RRT-OMP and $\text{OMP}_{\sigma^2}$ outperforms $\text{OMP}_{k_0}$. The performance of PaTh(OMP) follows that of RRT-OMP at low $k_0$, however, the performance of PaTh(OMP)  deteriorates significantly as $k_0$ increases. We next consider the performance of algorithms in Fig.\ref{fig.had512_rand} where $n$ and $p$ are high and signal is random.  TF-OMP performs  badly due to high $DR(\boldsymbol{\beta})$. RRT-OMP performs worse than PaTh(OMP), $\text{OMP}_{k_0}$ and $\text{OMP}_{\sigma^2}$ when SNR is low, i.e., 10dB and $k_0$ is high. When $k_0$ is low, RRT-OMP performs as good as $\text{OMP}_{k_0}$. At the  moderate 20dB SNR, RRT-OMP performs similar to or better than PaTh(OMP),  $\text{OMP}_{k_0}$ and $\text{OMP}_{\sigma^2}$ at all values of $k_0$. 
\begin{figure}[htb]
\includegraphics[width=\columnwidth]{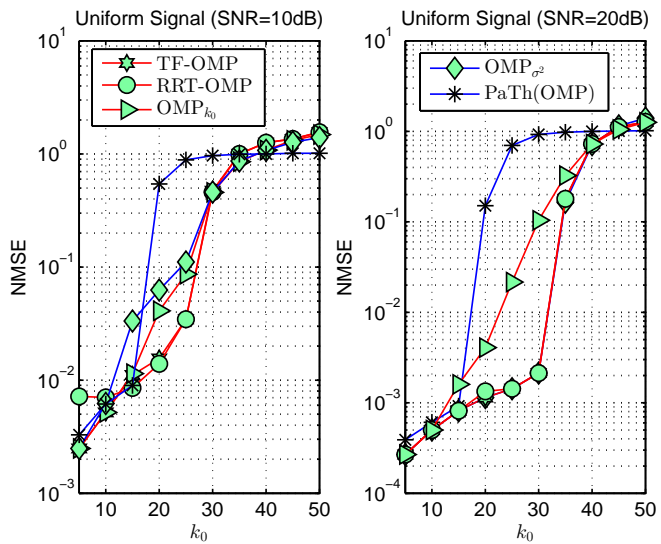}
\caption{NMSE performance: ${\bf X}$ random. $n=200$, $p=1000$.  }
\label{fig.n200_unif}
\end{figure}

\begin{figure}[htb]
\includegraphics[width=\columnwidth]{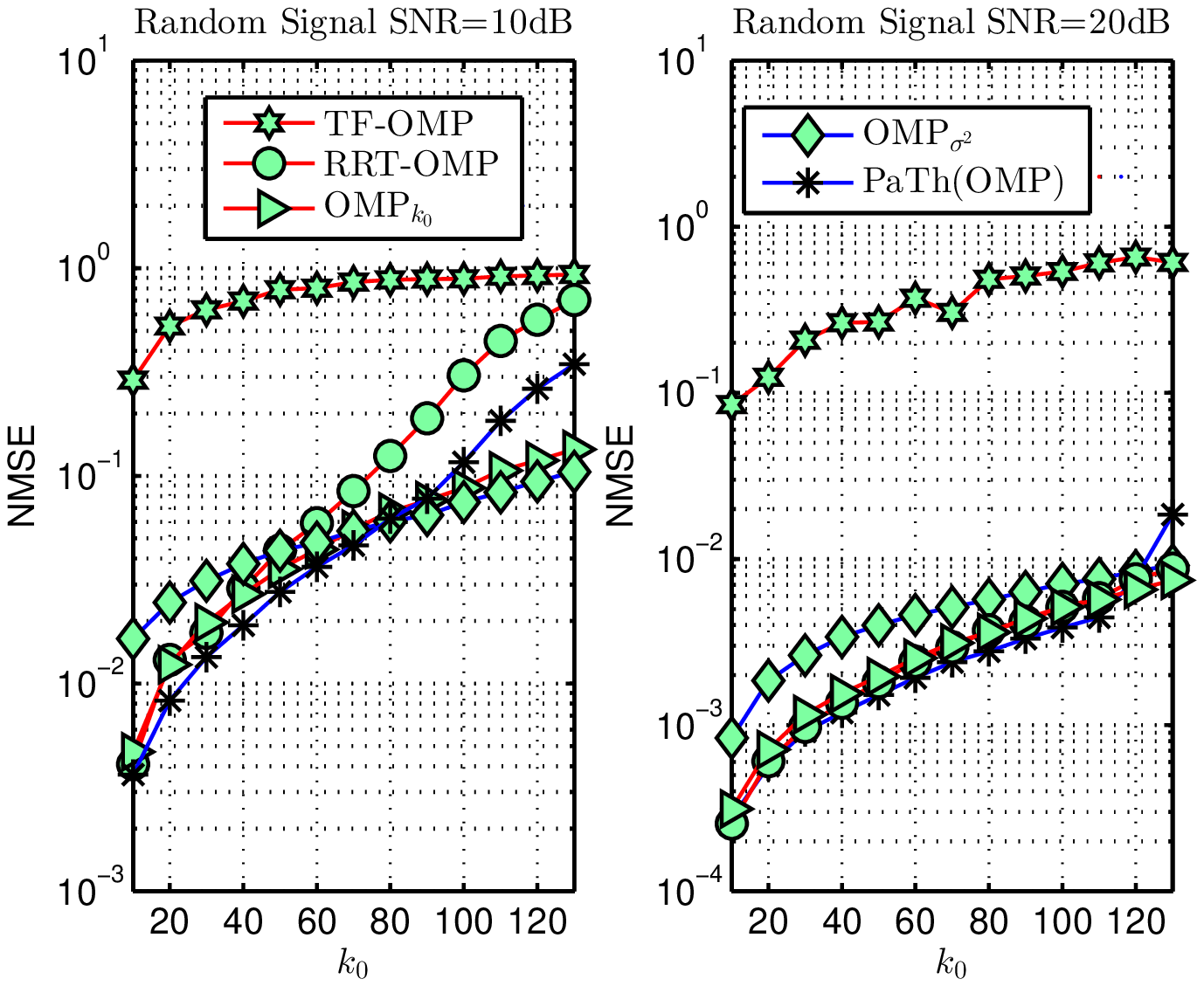}
\caption{NMSE performance ${\bf X}=[{\bf I}_{512},{\bf H}_{512}]$.  }
\label{fig.had512_rand}
\end{figure}

To summarize, the performance of TF-OMP and RRT-OMP are similar to or better than $\text{OMP}_{k_0}$ and  $\text{OMP}_{\sigma^2}$ at all SNR and significantly better than that of PaTh(OMP) when the signal model is uniform. When the signal model is random, the performance of TF-OMP degrades significantly, whereas, the performance of RRT-OMP closely matches the performance of $\text{OMP}_{k_0}$ and  $\text{OMP}_{\sigma^2}$ except when both $k_0$ is high and SNR is low.
\section{Conclusion and Future Research}
This article developed  two novel  frameworks  to achieve efficient sparse recovery using OMP and OLS algorithms when both sparsity $k_0$ and noise statistics $\{\sigma^2,\epsilon_2\}$ are unavailable. The performance of this framework is analysed both analytically and numerically.   The broader area of CS involves  many  scenarios other than the linear regression model considered in this article.  However, most CS algorithms involve tuning parameters that depends on nuisance parameters like $\sigma^2$ which are difficult to estimate. Hence, it is of tremendous importance to develop  SNO  and computationally efficient  algorithms like TF-IGP/RRT-IGP  for other CS applications also.
\section*{Appendix A:Proof of Theorem 1.}
\begin{proof}
{Reiterating, $k_*=\min\{k:\mathcal{I}\subseteq \hat{\mathcal{I}}_k\}$, where $\hat{\mathcal{I}}_k=Alg({\bf y},{\bf X},k)$ is the support estimate returned by `Alg' at sparsity level $k$.   $k_{max}$ is the maximum sparsity level of `Alg' in TABLE \ref{tab:tf-omp}. $RR(k)$ for $k\in \{k_*+1,\dotsc,k_{max}\}$ satisfies $RR(k)=\frac{\|({\bf I}_n-{\bf P}_{\hat{\mathcal{I}_k}}){\bf w}\|_2}{\|({\bf I}_n-{\bf P}_{\hat{\mathcal{I}}_{k-1}}){\bf w}\|_2}$, where ${\bf w}\sim \mathcal{N}(0,\sigma^2{\bf I}_n)$. Dividing both numerator and denominator of $RR(k)$ by $\sigma$ gives $RR(k)= \frac{\|({\bf I}_n-{\bf P}_{\hat{\mathcal{I}_k}}){\bf z}\|_2}{\|({\bf I}_n-{\bf P}_{\hat{\mathcal{I}}_k}){\bf z}\|_2}$, where ${\bf z}={\bf w}/\sigma\sim \mathcal{N}({\bf 0}_n,{\bf I}_n)$. B1) of Theorem \ref{thm:Beta} states that  $\mathbb{P}(\underset{k_*+1\leq k\leq k_{max}}{\min}RR(k)<\Gamma_{RRT}^{\alpha})\leq \alpha$. Here, $\Gamma_{RRT}^{\alpha}=\underset{k=1,\dotsc,k_{max}}{\min}\sqrt{F_{\frac{n-k}{2},0.5}^{-1}\left(\dfrac{\alpha}{k_{max}(p-k+1)}\right)}>0$ and $F^{-1}_{a,b}(x)$ is the inverse CDF of a $\mathbb{B}(a,b)$ R.V. Note that the Beta R.V is  bounded away from zero in probability. This implies that $\forall \alpha>0$, $F_{\frac{n-k}{2},0.5}^{-1}\left(\dfrac{\alpha}{k_{max}(p-k+1)}\right)>0$ and  $\Gamma_{RRT}^{\alpha}>0$. Hence,  B2) of Theorem 1 is a direct consequence of B1). Next we prove B1) of Theorem 1.}

    First of all note that $k_*$, $\hat{\mathcal{I}}_k$,   $RR(k)$ etc. are all R.V with unknown distribution. The proof of B1) follows by lower bounding  $RR(k)$ for $k>k_*$  using R.V with known distributions. We first consider the behaviour of $RR(k)$ when  signal $\boldsymbol{\beta}={\bf 0}_p$ ($\mathcal{I}=\phi$) in which case $k_*=1$.  Later we generalize the result to the case $\boldsymbol{\beta}\neq {\bf 0}_p$ in which case $k_*\geq k_0=|\mathcal{I}|$.  The crux of the proof is the following Lemma proved in (41) of \cite{tsp} using Result 5.3.7 of \cite{ravishanker2001first}.  
\begin{lemma}\label{beta}
Consider a  deterministic sequence of projection matrices $\{{\bf P}_k\}_{k=1}^{k_{max}}$ of rank $k$ projecting onto an increasing sequence of subspaces $\mathcal{S}_1\subset \mathcal{S}_2\dotsc \subset\mathcal{S}_{k_{max}}$. Then $\dfrac{\|({\bf I}_n-{\bf P}_k){\bf z}\|_2^2}{\|({\bf I}_n-{\bf P}_{k-1}){\bf z}\|_2^2}\sim  \mathbb{B}(\dfrac{n-k}{2},\dfrac{1}{2})$ whenever ${\bf z} \sim \mathcal{N}({\bf 0}_n,{\bf I}_n)$. 
\end{lemma} 
\subsection{Producing a lower bound for $RR(k)$ when $\boldsymbol{\beta}= {\bf 0}_p$.} 
 First consider running $Alg$ when $\boldsymbol{\beta}= {\bf 0}_p$, i.e., input of `Alg' is ${\bf y}={\bf w}$. By  the definition of $k_{max}$, all sub matrices ${\bf X}_{\hat{\mathcal{I}}_k}$ are full rank and hence  the projection matrices ${\bf P}_{\hat{\mathcal{I}}_k}$ have rank $k$. Further, by  the monotonicity of supports in IGP, ${\bf P}_{\hat{\mathcal{I}}_k}$  are projecting onto an increasing sequence of subspaces. However, the projection matrices in IGP are not fixed \textit{a priori}. Rather,  the indices are selected from the data itself making the exact computation of the distribution of $RR(k)$ extremely difficult. 

Consider the step $k-1$ of the IGP. Current support estimate $\hat{\mathcal{I}}_{k-1}$ is itself a R.V.  Let $\mathcal{L}_{k-1}\subseteq \{[p]/\hat{\mathcal{I}}_{k-1}\}$ represents the set of all all possible indices $l$ at stage $k-1$ such that ${\bf X}_{\hat{\mathcal{I}}_{k-1}\cup l}$ is full rank. Clearly,  $|\mathcal{L}_{k-1}|\leq p-|\mathcal{I}_{k-1}|=p-k+1$. Define the conditional R.V $Z_k^{l}|\hat{\mathcal{I}}_{k-1}=\frac{\|({\bf I}_n-{\bf P}_{\hat{\mathcal{I}}_{k-1}\cup l}){\bf z}\|_2^2}{\|({\bf I}_n-{\bf P}_{\hat{\mathcal{I}}k-1}){\bf z}\|_2^2}$ for $l \in \mathcal{L}_{k-1}$.  By Lemma \ref{beta},  $Z_k^{l}|\hat{\mathcal{I}}_{k-1}$  satisfies 
\begin{equation}
Z_k^{l}|\hat{\mathcal{I}}_{k-1} \sim \mathcal{B}\left(\frac{n-k}{2},\frac{1}{2}\right), \ \forall l \in \mathcal{L}_{k-1}\ \textit{and}\ k\geq 1.
\end{equation}
Since the index selected in the $k-1^{th}$ iteration belongs to  $\mathcal{L}_{k-1}$, it follows that conditioned on $\hat{\mathcal{I}}_{k-1}$,
\begin{equation}
\underset{l\in \mathcal{L}_{k-1}}{\min}\sqrt{Z_k^l|\hat{\mathcal{I}}_{k-1}}\leq RR(k). 
\end{equation}
Let $\delta_k=\sqrt{F_{\frac{n-k}{2},0.5}^{-1}\left(\frac{\alpha}{k_{max}(p-k+1)}\right)}$. By definition, $\Gamma_{RRT}^{\alpha}$  satisfies $\Gamma_{RRT}^{\alpha}=\underset{k}{\min}\ \delta_k\leq \delta_k,\forall k$. It follows that 
\begin{equation}\label{firstbound}
\begin{array}{ll}
\mathbb{P}(RR(k)<\delta_k|\hat{\mathcal{I}}_{k-1})&\leq \mathbb{P}(\underset{l\in \mathcal{L}_{k-1}}{\min}\sqrt{Z_k^l|\hat{\mathcal{I}}_{k-1}}<\delta_k) \\
& \overset{(a)}{\leq} \sum\limits_{l \in \mathcal{L}_{k-1}}\mathbb{P}({Z_k^l}|\hat{\mathcal{I}}_{k-1}<\delta_k^2)\overset{(b)}{\leq} \dfrac{\alpha}{k_{max}}
\end{array}
\end{equation}
(a) in Eqn.\ref{firstbound} follows from the union bound. By the definition of $\delta_k$, $\mathbb{P}({Z_k^l}<\delta_k^2)=\dfrac{\alpha}{k_{max}(p-k+1)}$. (b) follows from this and the fact that $|\mathcal{L}_{k-1}|\leq p-k+1$. Note that $\hat{\mathcal{I}}_{k-1}$ is a R.V itself. Let $\mathcal{S}_{k-1}$ represents the set of all possible values of $\hat{\mathcal{I}}_{k-1}$. Since, $\mathbb{P}(RR(k)<\delta_k|\hat{\mathcal{I}}_{k-1})\leq {\alpha}/{k_{max}}$ is independent of $\hat{\mathcal{I}}_{k-1}$, it follows from the law of total probability that
\begin{equation}\label{bound}
\begin{array}{ll}
\mathbb{P}(RR(k)<\delta_k)&=\sum\limits_{\tilde{\mathcal{I}}_{k-1} \in \mathcal{S}_{k-1}}\mathbb{P}(RR(k)<\delta_k|\tilde{\mathcal{I}}_{k-1})\mathbb{P}(\tilde{\mathcal{I}}_{k-1}) \\
&{\leq} \dfrac{\alpha}{k_{max}}\sum\limits_{\tilde{\mathcal{I}}_{k-1} \in \mathcal{S}_{k-1}}\mathbb{P}(\tilde{\mathcal{I}}_{k-1})=\dfrac{\alpha}{k_{max}}
\end{array}
\end{equation}
When $\boldsymbol{\beta}= {\bf 0}_p$, the  bound in (\ref{bound}) is valid for all $k\geq k_*= 1$. 
\subsection{Extension to the case when $\boldsymbol{\beta}\neq {\bf 0}_p$.}
When $\boldsymbol{\beta}\neq {\bf 0}_p$ and  $k\leq k_*$, numerator and denominator of $RR(k)=\frac{\|({\bf I}_n-{\bf P}_{\hat{\mathcal{I}_k}}){\bf y}\|_2}{\|({\bf I}_n-{\bf P}_{\hat{\mathcal{I}}_k}){\bf y}\|_2}$ contain signal terms. The presence of signal terms prevent the application of Lemma \ref{beta}. Hence, it is not necessary  that the bound (\ref{bound}) hold true for $k\leq k_*$. 
However, for $k>k_*$, the signal components ${\bf X}\boldsymbol{\beta}$ are vanished and $RR(k)$ returns to the form $RR(k)=\frac{\|({\bf I}_n-{\bf P}_{\hat{\mathcal{I}_{k}}}){\bf z}\|_2}{\|({\bf I}_n-{\bf P}_{\hat{\mathcal{I}}_{k-1}}){\bf z}\|_2}$, where ${\bf z}\sim\mathcal{N}({\bf 0}_n,{\bf I}_n)$. Hence, it is true that $\mathbb{P}(RR(k)<\delta_k)\leq \dfrac{\alpha}{k_{max}}$ for $k>k_*$ even when $\boldsymbol{\beta}\neq {\bf 0}_p$. It then follows that
\begin{equation}\label{finalbound}
\begin{array}{ll}
\mathbb{P}(\underset{k>k_*}{\min}RR(k)\leq \Gamma^{\alpha}_{RRT})&\overset{(a)}{\leq}\sum\limits_{k>k_*}\mathbb{P}(RR(k)\leq \Gamma^{\alpha}_{RRT})  \\
&\overset{(b)}{\leq}\sum\limits_{k>k_*} \mathbb{P}(RR(k)\leq \delta_k) \\&\overset{(c)}{\leq}  \dfrac{(k_{max}-k_*)\alpha}{k_{max}}\leq \alpha.
\end{array}
\end{equation}
(a) in (\ref{finalbound}) follows from the union bound, (b) follows from $\Gamma^{\alpha}_{RRT}=\underset{k=1,\dotsc,k_{max}}{\min}\delta_k$ and (c) follows from $\mathbb{P}(RR(k)<\delta_k)\leq \dfrac{\alpha}{k_{max}}$ for $k>k_*$. (\ref{finalbound}) proves B1) of  Theorem \ref{thm:Beta}. 
\end{proof}
 \section*{Appendix B: Proof of Theorem \ref{thm:tf}.}
\begin{proof}
{Theorem \ref{thm:tf} states that the  TF-IGP support estimate $\hat{\mathcal{I}}_{\hat{k}_{TF}}$, where $\hat{k}_{TF}=\underset{k}{\arg\min}\ RR(k)$  satisfies $\mathcal{I}\subseteq \hat{\mathcal{I}}_{\hat{k}_{TF}}$ with  $|\hat{\mathcal{I}}_{\hat{k}_{TF}}|\leq k_{sup}$, if $\|{\bf w}\|_2\leq \epsilon_2\leq \min(\epsilon_{sup},\epsilon_{sig},\epsilon_{\bf X})$. For this to happen, it is sufficient that E1)-E2) occurs simultaneously.  \\
E1).  $k_*=\min\{k: \mathcal{I}\subseteq\hat{\mathcal{I}}_k\}$ satisfies $k_*\leq k_{sup}$ and  \\
E2). $\hat{k}_{TF}=\underset{k}{\arg\min}\ RR(k)$ equals $k_*$. }\\
Assumption 1  implies that  the event E1) occurs if $\epsilon_2<\epsilon_{sup}$. Hence, it is sufficient to show that E2) is true, i.e., $RR(k)>RR(k_*)$ for $k<k_*$ and   $RR(k)>RR(k_*)$ for $k>k_*$ under the assumption that  the noise ${\bf w}$ satisfies $\epsilon_2 \leq \epsilon_{sup}$, i.e., E1) is true. First consider the condition  $RR(k)>RR(k_*)$ for $k<k_*$. The following bounds on $RR(k_*)$ and $RR(k)$ are derived in APPENDIX C. 
\begin{equation}\label{ub_tk0}
RR(k_*)=\dfrac{\|{\bf r}^{(k_*)}\|_2}{\|{\bf r}^{(k_*-1)}\|_2} \leq \dfrac{\epsilon_2}{\sqrt{1-\delta_{k_{sup}}}\boldsymbol{\beta}_{min}-\epsilon_2}, \forall \epsilon_2<\epsilon_{sup}. \ \ \ 
\end{equation}
\begin{equation}\label{lb_on_klessk0}
RR(k)\geq \dfrac{\sqrt{1-\delta_{k_{sup}}}\boldsymbol{\beta}_{min}-\epsilon_2}{\sqrt{1+\delta_{k_{sup}}}(\boldsymbol{\beta}_{max}+\boldsymbol{\beta}_{min})+\epsilon_2},\  \ \forall k<k_*\  \text{and}
\end{equation}
$\forall\epsilon_2<\epsilon_{sup}$. For $RR(k)>RR(k_*),\forall k<k_*$, it is sufficient  that the lower bound (\ref{lb_on_klessk0}) on $RR(k)$ for $k<k_*$ is larger than the upper bound (\ref{ub_tk0}) on $RR(k_*)$, i.e., 
\begin{equation}\label{ee1}
\dfrac{\sqrt{1-\delta_{k_{sup}}}\boldsymbol{\beta}_{min}-\epsilon_2}{\sqrt{1+\delta_{k_{sup}}}(\boldsymbol{\beta}_{max}+\boldsymbol{\beta}_{min})+\epsilon_2}\geq   \dfrac{\epsilon_2}{\sqrt{1-\delta_{k_{sup}}}\boldsymbol{\beta}_{min}-\epsilon_2}.
\end{equation}
(\ref{ee1})  is true whenever $\epsilon_2\leq \min(\epsilon_{sup},{\epsilon_{sig}})$, where 
\begin{equation}\label{epsb}
{\epsilon_{sig}}= \frac{\sqrt{1-\delta_{k_{sup}}}\boldsymbol{\beta}_{min}}{1+\frac{\sqrt{1+\delta_{k_{sup}}}}{\sqrt{1-\delta_{k_{sup}}}} \left(2+  \frac{\boldsymbol{\beta}_{max}}{\boldsymbol{\beta}_{min}} \right)}. 
\end{equation}

Next  consider the condition $RR(k)>RR(k_*),\forall k<k_*$   assuming  that $\epsilon_2<\epsilon_{sup}$, i.e.,  $\mathcal{I}\subseteq \hat{\mathcal{I}}_{k_*}$. By  Corollary \ref{corr:beta},  $ RR(k)> \Gamma_{Alg}({\bf X}),\forall k\geq k_*$ and $0<\Gamma_{Alg}({\bf X})\leq 1$  is a constant. At the same time,  the  bound (\ref{ub_tk0}) on $RR(k_*)$  is a decreasing function of $\epsilon_2$.  Hence, $\exists {\epsilon_{\bf X}}>0$ given by
\begin{equation}\label{epsilon_c}
{\epsilon_{\bf X}}=\dfrac{\sqrt{1-\delta_{k_{sup}}}\boldsymbol{\beta}_{min}\Gamma_{Alg}({\bf X})}{1+\Gamma_{Alg}({\bf X})}
\end{equation} 
such that   $RR(k_*)<RR(k),\forall k>k_*$ if $\epsilon_2<\min(\epsilon_{sup},{\epsilon_{\bf X}})$. Combining (\ref{epsb}) and (\ref{epsilon_c}), it follows that both E1) and E2) are satisfied  if $\epsilon_2<\min(\epsilon_{sup},{\epsilon_{sig}},{\epsilon_{\bf X}})$.   
\end{proof}
\section*{Appendix C: Bounds (\ref{ub_tk0}) and (\ref{lb_on_klessk0}) in Theorem \ref{thm:tf}.}
{ E1) in APPENDIX B states that $k_*=\min\{k:\mathcal{I}\subseteq \hat{\mathcal{I}}_k\}$ satisfies $k_0\leq k_*\leq k_{sup}$ whenever $\epsilon_2\leq \epsilon_{sup}$.  This implies that the 
support estimate $\hat{\mathcal{I}}_{k_*}=Alg({\bf X},{\bf y},k_*)$ is the ordered set $\{t_1,t_2,\dotsc,t_{k_*}\}$ such that $t_{k_*} \in \mathcal{I}$ and $\{t_1,t_2,\dotsc,t_{k_*-1}\}$ contains the rest $k_0-1$ entries in $\mathcal{I}$ and $k_*-k_0$ indices in $\mathcal{I}^C$.  Applying triangle inequality $\|{\bf a}+{\bf b}\|_2\leq \|{\bf a}\|_2+\|{\bf b}\|_2$, reverse triangle inequality $\|{\bf a}+{\bf b}\|_2\geq \|{\bf a}\|_2-\|{\bf b}\|_2$  and the bound $\|({\bf I}_n-{\bf P}_{\hat{\mathcal{I}_k}}){\bf w}\|_2\leq \|{\bf w}\|_2\leq \epsilon_2 $ to  $\|{\bf r}^{(k)}\|_2=\|({\bf I}_n-{\bf P}_{\hat{\mathcal{I}_k}}){\bf X}\boldsymbol{\beta}+({\bf I}_n-{\bf P}_{\hat{\mathcal{I}_k}}){\bf w}\|_2$  gives }
\begin{equation}\label{aaa1}
\|({\bf I}_n-{\bf P}_{\hat{\mathcal{I}_k}}){\bf X}\boldsymbol{\beta}\|_2-\epsilon_2\leq \|{\bf r}^{(k)}\|_2\leq \|({\bf I}_n-{\bf P}_{\hat{\mathcal{I}_k}}){\bf X}\boldsymbol{\beta}\|_2+\epsilon_2.
\end{equation}
Let ${u^k}=\mathcal{I}/\hat{\mathcal{I}}_k$  denotes the indices in $\mathcal{I}$ that are not selected after the $k^{th}$ iteration. 
Note that $({\bf I}_n-{\bf P}_{\hat{\mathcal{I}}_k}){\bf X}\boldsymbol{\beta}=({\bf I}_n-{\bf P}_{\hat{\mathcal{I}}_k}){\bf X}_{u^k}\boldsymbol{\beta}_{u^k}$. Further, $|\hat{\mathcal{I}}_k|+|u^k|\leq k_*\leq k_{sup}$ and $\hat{\mathcal{I}}_k\cap u^k=\phi$. Hence, by  a) and e) of  Lemma 1,
\begin{equation}\label{aaa2}
\sqrt{1-\delta_{k_{sup}}}\|\boldsymbol{\beta}_{u^k}\|_2\leq \|({\bf I}_n-{\bf P}_{\hat{\mathcal{I}}_k}){\bf X}_{u^k}\boldsymbol{\beta}_{u^k}\|_2\leq \sqrt{1+\delta_{k_{sup}}}\|\boldsymbol{\beta}_{u^k}\|_2.
\end{equation} 
Substituting (\ref{aaa2}) in (\ref{aaa1}) gives
\begin{equation}\label{Caibound}
\sqrt{1-\delta_{k_{sup}}}\|\boldsymbol{\beta}_{u^k}\|_2-\epsilon_2\leq\|{\bf r}^{(k)}\|_2\leq \sqrt{1+\delta_{k_{sup}}}\|\boldsymbol{\beta}_{u^k}\|_2+\epsilon_2.
\end{equation}
Since all $k_0$ non zero entries in $\boldsymbol{\beta}$ are selected after $k_*$ iterations, $u^{k_*}=\phi$  and hence $\|\boldsymbol{\beta}_{ u^{k_*}}\|_2=0$. Likewise, from the definition of $k_*=\min\{k:\mathcal{I}\subseteq \hat{\mathcal{I}}_k\}$, only one entry in $\boldsymbol{\beta}$ is left out after $k_*-1$ iterations. Hence, $|u^{k_*-1}|=1$ and  $\|\boldsymbol{\beta}_{u^{k_*-1}}\|_2\geq \boldsymbol{\beta}_{min}$. Substituting these values  in (\ref{Caibound}) gives $\|{\bf r}^{(k_*)}\|_2\leq \epsilon_2$ and $\|{\bf r}^{(k_*-1)}\|_2\geq \sqrt{1-\delta_{k_{sup}}}\boldsymbol{\beta}_{min}-\epsilon_2$. Hence, $RR(k_*)$ is  bounded by
\begin{equation*}
RR(k_*)=\dfrac{\|{\bf r}^{(k_*)}\|_2}{\|{\bf r}^{(k_*-1)}\|_2} \leq \dfrac{\epsilon_2}{\sqrt{1-\delta_{k_{sup}}}\boldsymbol{\beta}_{min}-\epsilon_2}, \forall \epsilon_2<\epsilon_{sup}.
\end{equation*}
which is the bound in (\ref{ub_tk0}). 
Next we lower bound $RR(k)$ for $k<k_*$. Note that $\boldsymbol{\beta}_{u^{k-1}}=\boldsymbol{\beta}_{u^{k}}+\boldsymbol{\beta}_{u^{k-1}/u^{k}}$
after appending enough zeros in appropriate locations. $\boldsymbol{\beta}_{u^{k-1}/u^{k}}$ can be either a vector of all zeros or can have one non zero entry.  Hence, $\|\boldsymbol{\beta}_{u^{k-1}/u^{k}}\|_2\leq \boldsymbol{\beta}_{max}$. Applying triangle inequality to $\boldsymbol{\beta}_{u^{k-1}}=\boldsymbol{\beta}_{u^{k}}+\boldsymbol{\beta}_{u^{k-1}/u^{k}}$ gives
the bound 
\begin{equation}\label{temp_bound}
\|\boldsymbol{\beta}_{u^{k-1}}\|_2\leq \|\boldsymbol{\beta}_{u^{k}}\|_2+\|\boldsymbol{\beta}_{u^{k-1}/u^{k}}\|_2 \leq  \|\boldsymbol{\beta}_{u^{k}}\|_2+  \boldsymbol{\beta}_{max}
\end{equation}
Applying (\ref{temp_bound}) and (\ref{Caibound}) in $RR(k)$ for $k<k_*$ gives
\begin{equation}\label{A1bound}
\begin{array}{ll}
RR(k)=\dfrac{\|{\bf r}^{(k)}\|_2}{\|{\bf r}^{(k-1)}\|_2} &\geq \dfrac{\sqrt{1-\delta_{k_{sup}}}\|\boldsymbol{\beta}_{u^k}\|_2-\epsilon_2}{\sqrt{1+\delta_{k_{sup}}}\|\boldsymbol{\beta}_{u^{k-1}}\|_2+\epsilon_2}\\
&\geq \dfrac{\sqrt{1-\delta_{k_{sup}}}\|\boldsymbol{\beta}_{u^k}\|_2-\epsilon_2}{\sqrt{1+\delta_{k_{sup}}}\left[\|\boldsymbol{\beta}_{u^{k}}\|_2+\boldsymbol{\beta}_{max}\right]+\epsilon_2}\\
\end{array}
\end{equation}
for all $\epsilon_2\leq \epsilon_{sup}$. The R.H.S of (\ref{A1bound}) can be rewritten as
\begin{equation}\label{A1bound2}
\begin{array}{ll}
\dfrac{\sqrt{1-\delta_{k_{sup}}}\|\boldsymbol{\beta}_{u^k}\|_2-\epsilon_2}{\sqrt{1+\delta_{k_{sup}}}\left[\|\boldsymbol{\beta}_{u^{k}}\|_2+\boldsymbol{\beta}_{max}\right]+\epsilon_2}=\dfrac{\sqrt{1-\delta_{k_{sup}}}}{\sqrt{1+\delta_{k_{sup}}}} \\
\ \ \ \ \    \times \left(1-\dfrac{\dfrac{\epsilon_2}{\sqrt{1-\delta_{k_{sup}}}}+\dfrac{\epsilon_2}{\sqrt{1+\delta_{k_{sup}}}}+\boldsymbol{\beta}_{max}}{\|\boldsymbol{\beta}_{u^{k}}\|_2+\boldsymbol{\beta}_{max}+\dfrac{\epsilon_2}{\sqrt{1+\delta_{k_{sup}}}}}\right)
\end{array}
\end{equation}
From (\ref{A1bound2}) it is clear that the R.H.S of (\ref{A1bound}) decreases with decreasing $\|\boldsymbol{\beta}_{u^k}\|_2$.  Note  that the minimum value of $\|\boldsymbol{\beta}_{u^{k}}\|_2$ is $\boldsymbol{\beta}_{min}$ itself. Hence, substituting $\|\boldsymbol{\beta}_{u^{k}}\|_2\geq \boldsymbol{\beta}_{min}$ in (\ref{A1bound})  will give the following bound in (\ref{lb_on_klessk0}).
\begin{equation*}\label{lb_on_klessk0_temp}
RR(k)\geq \dfrac{\sqrt{1-\delta_{k_{sup}}}\boldsymbol{\beta}_{min}-\epsilon_2}{\sqrt{1+\delta_{k_{sup}}}(\boldsymbol{\beta}_{max}+\boldsymbol{\beta}_{min})+\epsilon_2},\  \ \forall k<k_*\  \text{and} \  \epsilon_2<\epsilon_{sup}.
\end{equation*}
\section*{Appendix D: Proof of Theorem \ref{thm-rrt}.}
\begin{proof}
{Theorem \ref{thm-rrt} states that the   support estimate $\hat{\mathcal{I}}_{\hat{k}_{RRT}}$, where $\hat{k}_{RRT}=\max\{k:RR(k)<\Gamma_{Alg}^{lb}({\bf X})\}$  satisfies $\mathcal{I} \subseteq \hat{\mathcal{I}}_{\hat{k}_{RRT}}$ and $|\hat{\mathcal{I}}_{\hat{k}_{RRT}}|\leq k_{sup}$, if $\|{\bf w}\|_2\leq \epsilon_2\leq \min(\epsilon_{sup},\epsilon_{RRT})$. For this to happen, it is sufficient  that   $k_*=\min\{k: \mathcal{I}\subseteq \hat{\mathcal{I}}_{k}\}\leq k_{sup}$  and $\hat{k}_{RRT}=\max\{k:RR(k)<\Gamma_{Alg}^{lb}({\bf X})\}$ equals $k_*$. } By Assumption 1, $k_*\leq k_{sup}$ whenever $\epsilon_2<\epsilon_{sup}$. Note that $\hat{k}_{RRT}=k_*$, iff  $RR(k_*)<\Gamma_{Alg}^{lb}({\bf X})$ and $RR(k)\geq \Gamma_{Alg}^{lb}({\bf X})$ for $k>k_*$. By the definition of $\Gamma_{Alg}^{lb}({\bf X})$,  $\Gamma_{Alg}^{lb}({\bf X})\leq \Gamma_{Alg}({\bf X})< RR(k)$ for $k>k_*$ at all SNR. Hence, to show $\hat{k}_{RRT}=k_*$, one only need  to show that $RR(k_*)<\Gamma_{Alg}^{lb}({\bf X})$. From (\ref{ub_tk0}), $RR(k_*)$ satisfy $RR(k_*)\leq \dfrac{\epsilon_2}{\sqrt{1-\delta_{k_{sup}}}\boldsymbol{\beta}_{min}-\epsilon_2}$, whenever $\epsilon_2<\epsilon_{sup}$. Thus, $RR(k_*)<\Gamma_{Alg}^{lb}({\bf X})$ if $\epsilon_2\leq  \min(\epsilon_{sup},{\epsilon}_{RRT})$, where
\begin{equation}
{\epsilon}_{RRT}=\dfrac{\sqrt{1-\delta_{k_{sup}}}\boldsymbol{\beta}_{min}\Gamma_{Alg}^{lb}({\bf X})}{1+\Gamma_{Alg}^{lb}({\bf X})}.
\end{equation}
Hence,  the support estimate $\hat{\mathcal{I}}_{\hat{k}_{RRT}}=\hat{\mathcal{I}}_{k_*}$  satisfies $\mathcal{I}\subseteq \hat{\mathcal{I}}_{\hat{k}_{RRT}}$ with $|\hat{\mathcal{I}}_{\hat{k}_{RRT}}|=k_*\leq k_{sup}$, if $\epsilon_2<\min(\epsilon_{sup},{\epsilon_{RRT}})$. 
\end{proof}

\section*{Appendix E: Bound (\ref{mismatch1}) in Section \rom{6}.C. }
Let $\hat{\mathcal{I}}_{k_0}=\text{OMP}({\bf y},{\bf X},k_0)$ be the support estimated by $\text{OMP}_{k_0}$. Then the corresponding estimate $\hat{\boldsymbol{\beta}}$ satisfies  $\hat{\boldsymbol{\beta}}_{\hat{\mathcal{I}}_{k_0}}={\bf X}_{\hat{\mathcal{I}}_{k_0}}^{\dagger}{\bf y}$ and $\hat{\boldsymbol{\beta}}_{\hat{\mathcal{I}}_{k_0}^C}={\bf 0}_{p-k_0}$. $\overline{\bf a}_{\mathcal{J}}$ denotes the vector in $\mathbb{R}^p$ such that $\overline{\bf a}_{\mathcal{J}}(i)={\bf a}_i$ for $i\in\mathcal{J}$ and $\overline{\bf a}_{\mathcal{J}}(i)=0$ for $i\notin\mathcal{J}$. $\hat{\boldsymbol{\beta}}$ satisfies  $\boldsymbol{\beta}-\hat{\boldsymbol{\beta}}=
\overline{\boldsymbol{\beta}_{\mathcal{I}/\hat{\mathcal{I}}_{k_0}}}+
\overline{\boldsymbol{\beta}_{\hat{\mathcal{I}}_{k_0}}-\hat{\boldsymbol{\beta}}_{\hat{\mathcal{I}}_{k_0}}\   }$. 
For signals $\boldsymbol{\beta}\in \mathcal{B}_1$, $|\mathcal{I}/\hat{\mathcal{I}}_{k_0}|>0$. Then, the following bounds hold true.
\begin{equation}\label{mismatch}
\begin{array}{ll}
\|\boldsymbol{\beta}-\hat{\boldsymbol{\beta}}\|_2&\overset{(a)}{\geq} \|\boldsymbol{\beta}_{\mathcal{I}/\hat{\mathcal{I}}_{k_0}}\|_2- \|\boldsymbol{\beta}_{\hat{\mathcal{I}}_{k_0}}-\hat{\boldsymbol{\beta}}_{\hat{\mathcal{I}}_{k_0}}\|_2 \\
&=\|\boldsymbol{\beta}_{\mathcal{I}/\hat{\mathcal{I}}_{k_0}}\|_2-\|\boldsymbol{\beta}_{\hat{\mathcal{I}}_{k_0}}-{\bf X}_{\hat{\mathcal{I}}_{k_0}}^{\dagger}({\bf X}_{\mathcal{I}}\boldsymbol{\beta}_{\mathcal{I}}+{\bf w})\|_2 \\
&\overset{(b)}{\geq} \|\boldsymbol{\beta}_{\mathcal{I}/\hat{\mathcal{I}}_{k_0}}\|_2-\|\boldsymbol{\beta}_{\hat{\mathcal{I}}_{k_0}}-{\bf X}_{\hat{\mathcal{I}}_{k_0}}^{\dagger}{\bf X}_{\hat{\mathcal{I}}_{k_0}\cap \mathcal{I}}\boldsymbol{\beta}_{\hat{\mathcal{I}}_{k_0}\cap \mathcal{I}} \|_2 \\
& \  -\|{\bf X}_{\hat{\mathcal{I}}_{k_0}}^{\dagger}{\bf X}_{{\mathcal{I}}/ \hat{\mathcal{I}}_{k_0}}\boldsymbol{\beta}_{{\mathcal{I}}/ \hat{\mathcal{I}}_{k_0}} \|_2-\|{\bf X}_{\hat{\mathcal{I}}_{k_0}}^{\dagger}{\bf w}\|_2 \\
\end{array}
\end{equation} 
 (a)  in (\ref{mismatch}) follows from reverse triangle inequality $\|{\bf a}+{\bf b}\|_2\geq \|{\bf a}\|_2-\|{\bf b}\|_2$ and b) follows from the triangle inequality $\|{\bf a}+{\bf b}\|_2\leq \|{\bf a}\|_2+\|{\bf b}\|_2$. Further, ${\bf X}_{\hat{\mathcal{I}}_{k_0}}^{\dagger}{\bf X}_{\hat{\mathcal{I}}_{k_0}\cap \mathcal{I}}\boldsymbol{\beta}_{\hat{\mathcal{I}}_{k_0}\cap \mathcal{I}}={\bf X}_{\hat{\mathcal{I}}_{k_0}}^{\dagger}{\bf X}_{\hat{\mathcal{I}}_{k_0}}[\boldsymbol{\beta}_{\hat{\mathcal{I}}_{k_0}\cap \mathcal{I}}^T,{\bf 0}_{k_0-|\hat{\mathcal{I}}_{k_0}\cap \mathcal{I}|}^T]^T=[\boldsymbol{\beta}_{\hat{\mathcal{I}}_{k_0}\cap \mathcal{I}}^T,{\bf 0}_{k_0-|\hat{\mathcal{I}}_{k_0}\cap \mathcal{I}|}^T]^T$. Hence, $\boldsymbol{\beta}_{\hat{\mathcal{I}}_{k_0}}-{\bf X}_{\hat{\mathcal{I}}_{k_0}}^{\dagger}{\bf X}_{\hat{\mathcal{I}}_{k_0}\cap \mathcal{I}}\boldsymbol{\beta}_{\hat{\mathcal{I}}_{k_0}\cap \mathcal{I}}={\bf 0}_{k_0}$. Note that the sets $\hat{\mathcal{I}}_{k_0}$ and $\mathcal{I}/\hat{\mathcal{I}}_{k_0}$ are disjoint and $|\hat{\mathcal{I}}_{k_0}|+|\mathcal{I}/\hat{\mathcal{I}}_{k_0}|\leq 2k_0$. The following bounds thus follows from b) and d) of Lemma 1.
\begin{equation*}
\begin{array}{ll}
\|{\bf X}_{\hat{\mathcal{I}}_{k_0}}^{\dagger}{\bf X}_{{\mathcal{I}}/ \hat{\mathcal{I}}_{k_0}}\boldsymbol{\beta}_{{\mathcal{I}}/ \hat{\mathcal{I}}_{k_0}} \|_2&=\|\left({\bf X}_{\hat{\mathcal{I}}_{k_0}}^{T}{\bf X}_{\hat{\mathcal{I}}_{k_0}}\right)^{-1}{\bf X}_{\hat{\mathcal{I}}_{k_0}}^{T}{\bf X}_{{\mathcal{I}}/ \hat{\mathcal{I}}_{k_0}}\boldsymbol{\beta}_{{\mathcal{I}}/ \hat{\mathcal{I}}_{k_0}} \|_2 \\
&\leq \dfrac{1}{1-\delta_{k_0}}\|{\bf X}_{\hat{\mathcal{I}}_{k_0}}^{T}{\bf X}_{{\mathcal{I}}/ \hat{\mathcal{I}}_{k_0}}\boldsymbol{\beta}_{{\mathcal{I}}/ \hat{\mathcal{I}}_{k_0}} \|_2 \\
&\leq \dfrac{\delta_{2k_0}}{1-\delta_{k_0}}\|\boldsymbol{\beta}_{\mathcal{I}/\hat{\mathcal{I}}_{k_0}}\|_2.
\end{array}
\end{equation*}
Also by c) of Lemma 1, $\|{\bf X}_{\hat{\mathcal{I}}_{k_0}}^{\dagger}{\bf w}\|_2\leq \dfrac{\epsilon_2}{\sqrt{1-\delta_{k_0}}}$. Substituting these results in (\ref{mismatch}) gives
$ \|\boldsymbol{\beta}-\hat{\boldsymbol{\beta}}\|_2{\geq}(1-\dfrac{\delta_{2k_0}}{1-\delta_{k_0}})\|\boldsymbol{\beta}_{\mathcal{I}/\hat{\mathcal{I}}_{k_0}}\|_2-\dfrac{\epsilon_2}{\sqrt{1-\delta_{k_0}}}$. Bound (\ref{mismatch1}) then follows from $\|\boldsymbol{\beta}_{\mathcal{I}/\hat{\mathcal{I}}_{k_0}}\|_2\geq \boldsymbol{\beta}_{min}$. 

\bibliography{compressive.bib}

\begin{thebibliography}{10}
\providecommand{\url}[1]{#1}
\csname url@samestyle\endcsname
\providecommand{\newblock}{\relax}
\providecommand{\bibinfo}[2]{#2}
\providecommand{\BIBentrySTDinterwordspacing}{\spaceskip=0pt\relax}
\providecommand{\BIBentryALTinterwordstretchfactor}{4}
\providecommand{\BIBentryALTinterwordspacing}{\spaceskip=\fontdimen2\font plus
\BIBentryALTinterwordstretchfactor\fontdimen3\font minus
  \fontdimen4\font\relax}
\providecommand{\BIBforeignlanguage}[2]{{%
\expandafter\ifx\csname l@#1\endcsname\relax
\typeout{** WARNING: IEEEtran.bst: No hyphenation pattern has been}%
\typeout{** loaded for the language `#1'. Using the pattern for}%
\typeout{** the default language instead.}%
\else
\language=\csname l@#1\endcsname
\fi
#2}}
\providecommand{\BIBdecl}{\relax}
\BIBdecl

\bibitem{kallummil2017tuning}
S.~Kallummil and S.~Kalyani, ``Tuning free orthogonal matching pursuit,''
  \emph{arXiv preprint arXiv:1703.05080}, 2017.

\bibitem{eldar2012compressed}
Y.~C. Eldar and G.~Kutyniok, \emph{Compressed sensing: {T}heory and
  applications}.\hskip 1em plus 0.5em minus 0.4em\relax Cambridge University
  Press, 2012.

\bibitem{single_snap}
S.~Fortunati, R.~Grasso, F.~Gini, and M.~S. Greco, ``Single snapshot {DOA}
  estimation using compressed sensing,'' in \emph{Proc.ICASSP}, May 2014, pp.
  2297--2301.

\bibitem{MLOMP}
S.~Kallummil and S.~Kalyani, ``Combining {ML} and compressive sensing:
  {D}etection schemes for generalized space shift keying,'' \emph{IEEE Wireless
  Commun. Lett.}, vol.~5, no.~1, pp. 72--75, Feb 2016.

\bibitem{tropp2006just}
J.~Tropp, ``Just relax: {C}onvex programming methods for identifying sparse
  signals in noise,'' \emph{IEEE Trans. Inf. Theory}, vol.~52, no.~3, pp.
  1030--1051, March 2006.

\bibitem{candes2007dantzig}
T.~T. Emmanuel~Candes, ``The {D}antzig selector: {S}tatistical estimation when
  p is much larger than n,'' \emph{Ann. Stat.}, vol.~35, no.~6, pp. 2313--2351,
  2007.

\bibitem{subspacepursuit}
W.~Dai and O.~Milenkovic, ``Subspace pursuit for compressive sensing signal
  reconstruction,'' \emph{IEEE Trans. Inf. Theory}, vol.~55, no.~5, pp.
  2230--2249, May 2009.

\bibitem{cosamp}
D.~Needell and J.~A. Tropp, ``Co{S}a{MP}: {I}terative signal recovery from
  incomplete and inaccurate samples,'' \emph{Appl. Comput. Harmon. Anal.},
  vol.~26, no.~3, pp. 301 -- 321, 2009.

\bibitem{tropp2004greed}
J.~A. Tropp, ``Greed is good: {A}lgorithmic results for sparse approximation,''
  \emph{IEEE Trans. Inf. Theory}, vol.~50, no.~10, pp. 2231--2242, 2004.

\bibitem{cai2011orthogonal}
T.~Cai and L.~Wang, ``Orthogonal matching pursuit for sparse signal recovery
  with noise,'' \emph{IEEE Trans. Inf. Theory}, vol.~57, no.~7, pp. 4680--4688,
  July 2011.

\bibitem{OMP_wang}
J.~Wang, ``Support recovery with orthogonal matching pursuit in the presence of
  noise,'' \emph{IEEE Trans. Signal Process.}, vol.~63, no.~21, pp. 5868--5877,
  Nov 2015.

\bibitem{extra}
J.~Wang and B.~Shim, ``Exact recovery of sparse signals using orthogonal
  matching pursuit: {H}ow many iterations do we need?'' \emph{IEEE Trans.
  Signal Process.}, vol.~64, no.~16, pp. 4194--4202, Aug 2016.

\bibitem{omp_sharp}
J.~Wen, Z.~Zhou, J.~Wang, X.~Tang, and Q.~Mo, ``A sharp condition for exact
  support recovery of sparse signals with orthogonal matching pursuit,'' in
  \emph{Proc. ISIT}, July 2016, pp. 2364--2368.

\bibitem{omp_necess}
C.~Liu, F.~Yong, and J.~Liu, ``Some new results about sufficient conditions for
  exact support recovery of sparse signals via orthogonal matching pursuit,''
  \emph{IEEE Trans. Signal Process.}, vol.~PP, no.~99, pp. 1--1, 2017.

\bibitem{prateek}
P.~Jain, A.~Tewari, and I.~S. Dhillon, ``Partial hard thresholding,''
  \emph{IEEE Trans. Inf. Theory}, vol.~63, no.~5, pp. 3029--3038, May 2017.

\bibitem{ERC-OMP}
C.~Soussen, R.~Gribonval, J.~Idier, and C.~Herzet, ``Joint {K}-{S}tep analysis
  of orthogonal matching pursuit and orthogonal least squares,'' \emph{IEEE
  Trans. Inf. Theory}, vol.~59, no.~5, pp. 3158--3174, May 2013.

\bibitem{OLSarxiv}
\BIBentryALTinterwordspacing
J.~Wen, J.~Wang, and Q.~Zhang, ``Necessary and sufficient conditions for
  orthogonal least squares,'' \emph{CoRR}, vol. abs/1611.07628, 2016. [Online].
  Available: \url{http://arxiv.org/abs/1611.07628}
\BIBentrySTDinterwordspacing

\bibitem{giraud2012}
C.~Giraud, S.~Huet, and N.~Verzelen, ``High-dimensional regression with unknown
  variance,'' \emph{Statist. Sci.}, vol.~27, no.~4, pp. 500--518, 11 2012.

\bibitem{sqlasso}
A.~Belloni, V.~Chernozhukov, and L.~Wang, ``Square-root {LASSO}: {P}ivotal
  recovery of sparse signals via conic programming,'' \emph{Biometrika},
  vol.~98, no.~4, p. 791, 2011.

\bibitem{spice}
P.~Stoica, P.~Babu, and J.~Li, ``{SPICE}: {A} sparse covariance-based
  estimation method for array processing,'' \emph{IEEE Trans. Signal Process.},
  vol.~59, no.~2, pp. 629--638, Feb 2011.

\bibitem{spice_like}
P.~Stoica and P.~Babu, ``{SPICE} and {LIKES}: {T}wo hyperparameter-free methods
  for sparse-parameter estimation,'' \emph{Signal Processing}, vol.~92, no.~7,
  pp. 1580 -- 1590, 2012.

\bibitem{spicenote}
C.~R. Rojas, D.~Katselis, and H.~Hjalmarsson, ``A note on the {SPICE} method,''
  \emph{IEEE Trans. Signal Process.}, vol.~61, no.~18, pp. 4545--4551, Sept
  2013.

\bibitem{vats2014path}
D.~Vats and R.~Baraniuk, ``Path thresholding: {A}symptotically tuning-free
  high-dimensional sparse regression,'' in \emph{Artificial Intelligence and
  Statistics}, 2014, pp. 948--957.

\bibitem{ravishanker2001first}
N.~Ravishanker and D.~K. Dey, \emph{A first course in linear model
  theory}.\hskip 1em plus 0.5em minus 0.4em\relax CRC Press, 2001.

\bibitem{omp_rip_noise}
R.~Wu, W.~Huang, and D.~R. Chen, ``The exact support recovery of sparse signals
  with noise via orthogonal matching pursuit,'' \emph{IEEE Signal Process.
  Lett.}, vol.~20, no.~4, pp. 403--406, April 2013.

\bibitem{resdif}
G.~Dziwoki, ``Averaged properties of the residual error in sparse signal
  reconstruction,'' \emph{IEEE Signal Process. Lett.}, vol.~23, no.~9, pp.
  1170--1173, Sept 2016.

\bibitem{Xiong2014}
W.~Xiong, J.~Cao, and S.~Li, ``Sparse signal recovery with unknown signal
  sparsity,'' \emph{EURASIP. J. Adv. Signal Process.}, vol. 2014, no.~1, p.
  178, 2014.

\bibitem{tsp}
S.~Kallummil and S.~Kalyani, ``High {SNR} consistent linear model order
  selection and subset selection,'' \emph{IEEE Trans. Signal Process.},
  vol.~64, no.~16, pp. 4307--4322, Aug 2016.

\bibitem{elad_book}
M.~Elad, \emph{Sparse and Redundant Representations: {F}rom Theory to
  Applications in Signal and Image Processing}.\hskip 1em plus 0.5em minus
  0.4em\relax Springer, 2010.

\bibitem{cs_tsp}
S.~Kallummil and S.~Kalyani, ``High {SNR} consistent compressive sensing,''
  \emph{arXiv preprint arXiv:1703.03596}, 2017.

\end{thebibliography}
\bibliographystyle{IEEEtran}

\end{document}